\documentclass[twoside]{article}

\usepackage[preprint]{aistats2026}
\usepackage[round]{natbib}

\usepackage{graphicx}
\usepackage{algorithm}
\usepackage[noend]{algpseudocode}
\usepackage{physics}
\usepackage{subcaption}
\usepackage{booktabs}
\usepackage{arydshln} 
\usepackage{hyperref}
\usepackage{multirow}

\usepackage{amsmath,amsfonts,bm}
\usepackage{amsthm}

\def\1{\bm{1}}

\DeclareMathAlphabet{\mathsfit}{\encodingdefault}{\sfdefault}{m}{sl}
\SetMathAlphabet{\mathsfit}{bold}{\encodingdefault}{\sfdefault}{bx}{n}

\newcommand{\Var}{\mathrm{Var}}

\newcommand{\ourmethod}{FM-$\mathring{\Delta}$}
\newcommand{\ourname}{Simplex-to-Euclidean Flow Matching}

\newtheorem{proposition}{Proposition}
\newtheorem{lemma}{Lemma}
\newtheorem{theorem}{Theorem}

\begin{document}

\twocolumn[

\aistatstitle{
Simplex-to-Euclidean Bijections for Categorical Flow Matching
}

\aistatsauthor{ Bernardo Williams$^1$ \And  Victor M. Yeom-Song$^2, ^3$ \And Marcelo Hartmann$^1$ \And  Arto Klami$^1$ }

\aistatsaddress{$^1$Department of Computer Science, University of Helsinki \\ $^2$Department of Computer Science, Aalto University \\ $^3$ELLIS Institute Finland } ]

\begin{abstract}
We propose a method for learning and sampling from probability distributions supported on the simplex. Our approach maps the open simplex to Euclidean space via smooth bijections, leveraging the Aitchison geometry to define the mappings, and supports modeling categorical data by a Dirichlet interpolation that dequantizes discrete observations into continuous ones. This enables density modeling in Euclidean space through the bijection while still allowing exact recovery of the original discrete distribution. Compared to previous methods that operate on the simplex using Riemannian geometry or custom noise processes, our approach works in Euclidean space while respecting the Aitchison geometry, and achieves competitive performance on both synthetic and real-world data sets.
\end{abstract}

\section{INTRODUCTION}
\label{sec:intro}
We study the problem of learning and generating samples from probability distributions supported on the unit simplex, a setting that naturally arises when working with compositional data (vectors of non-negative components that sum 1o one). Learning distributions on the simplex provides a natural framework for many real-world applications. For example, in computational biology, sequences can be encoded as categorical variables e.g.\ for conditional generation of  DNA sequences \citep{Avdeyev2023} or proteins \citep{Campbell2024}, and compositional data naturally arises in geology, chemistry, design and economics \citep{Aitchison1982, Chereddy2025, Diederen2025}.

There are two main ways of modeling categorical data. The first are discrete-state models that manipulate directly the categorical states, often by modeling transition dynamics. This family includes discrete flow and diffusion models~\citep{Austin2021, Campbell2024, Gat2024, Sahoo2025} as well as masked diffusion models~\citep{Sahoo2024}. The other category adapts continuous-state models, such as standard diffusion and flow models, to work with categorical observations by continuous relaxations. We work within the latter category, to ease the use of established continuous models with well-understood learning dynamics and good implementations \citep{Karras2022,Lipman2024} for discrete data.

The continuous relaxation models can be further divided into two broad categories: Some operate directly on the simplex, while others avoid the simplex entirely by working in the ambient space $\mathbb{R}^K$, gradually moving continuous samples toward the vertices to recover discreteness~\citep{Chen2023, Eijkelboom2024}. Within the simplex-based models, two challenges arise: (i) handling the boundary, where discrete data lie, and (ii) accounting for the simplex’s non-Euclidean geometry. Existing approaches tackle these in different ways: Some define distributions directly on the simplex, either via custom marginals for diffusion processes~\citep{Avdeyev2023, Floto2023, Tae2025} or vector fields~\citep{Stark2024, Dunn2024, Tang2025}, or by mapping the simplex to the sphere through bijections~\citep{Davis2024, Cheng2024, Cheng2025}. We build on the methods that operate on the simplex, focusing on two related aspects: How to account for the geometry of the simplex and the discrete data at its boundary in a principled way.

\begin{figure}
    \centering
    \includegraphics[width=0.99\linewidth]{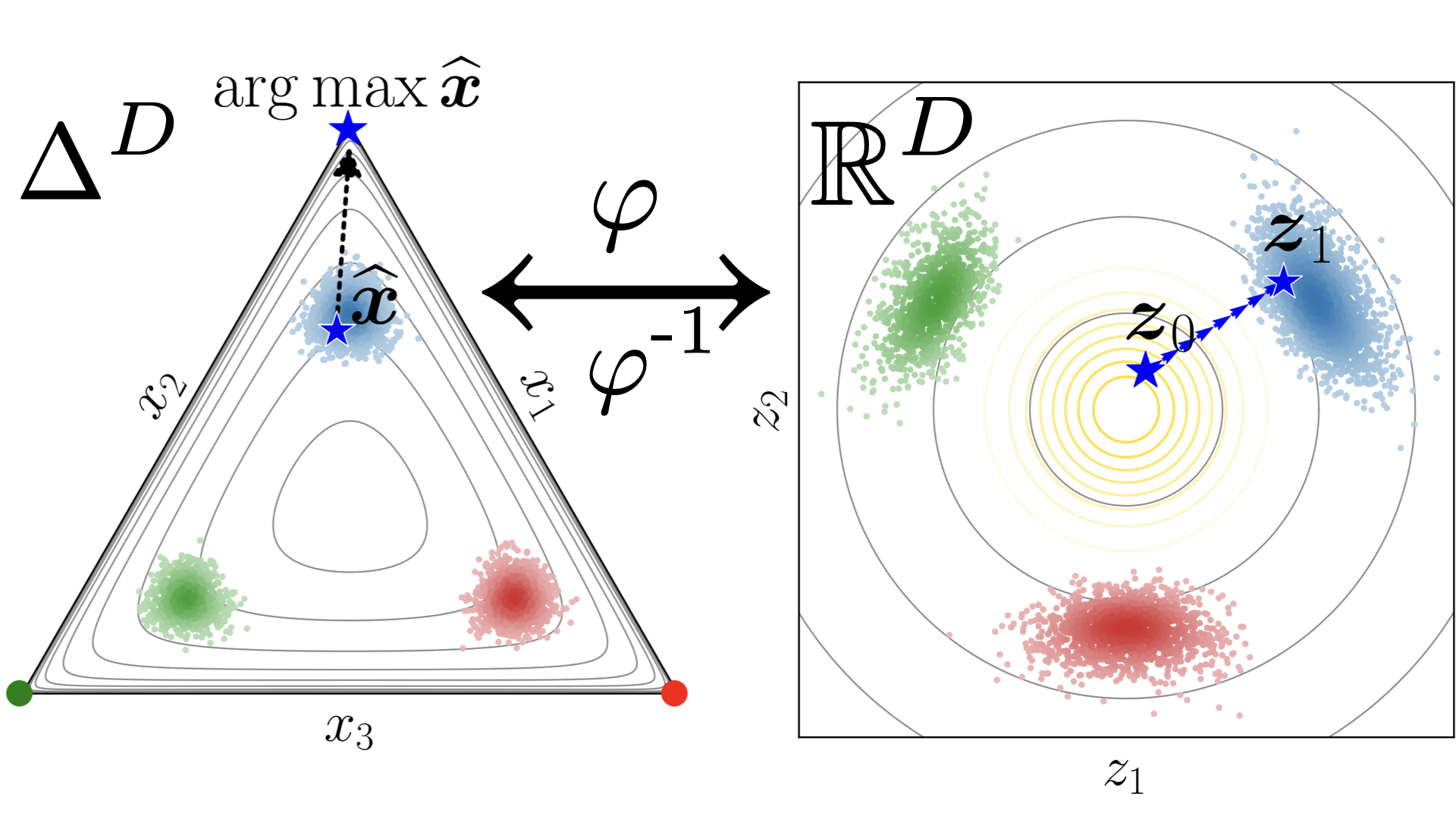}
    \caption{We stochastically interpolate categorical observations (color) to distributions on the interior of a simplex (left). The resulting Dirichlet mixture is transformed to Euclidean space (right) with a bijection ${\varphi}$, enabling use of standard continuous generative models, like conditional flow matching. Discrete samples are obtained by composition of the inverse transformation ${\varphi}^{-1}$ and $\arg\max$ operation.
    }
    \label{fig:fig1}
\end{figure}

Concretely, we propose a new approach for modeling categorical data with continuous generative models, aiming for conceptual and implementation simplicity to best leverage the existing continuous tools. We map the \emph{interior} of the simplex to Euclidean space using a smooth bijection drawn from compositional data analysis and the Aitchison geometry, a geometry induced by the logratio of the components \citep{Aitchison1982}. We can then train a standard continuous generative model in Euclidean space. For compositional data this is enough to account for the simplex geometry, but 
for categorical observations that lie at the boundary of the simplex we need additional tools. They are not covered by our bijection, but for training the model we can map them into interior points via a Dirichlet interpolation scheme, a stochastic generalization of a recent dequantization approach by \citet{Chereddy2025}. We do this so that we can retrieve the original category by a simple $\arg\max$-operation, which means we can also easily convert continuous samples into discrete ones.
Figure~\ref{fig:fig1} illustrates the method.

We demonstrate the approach using Flow Matching \citep{Lipman2023, Albergo2023} as the continuous generative model, and propose two alternative bijections: the stick-breaking transform \citep{Aitchison1982} and the isometric logratio transform \citep{Egozcue2003}. Both define a geometry on the simplex, and explicitly connect this geometry to Euclidean space. Moreover, they are computationally lightweight and easy to implement.

We establish several theoretical results and empirically validate the method on standard benchmarks, demonstrating highly competitive performance within the continuous relaxations and outperforming discrete-state models in low-dimensional problems.

\section{BACKGROUND}
\label{sec:prelim}
For $K$ classes, the simplex is a space of $K-1$ dimensions. For notational simplicity, we denote this by $D=K-1$.
Discrete data is denoted $\boldsymbol{c}\in\{ \boldsymbol{e}_1,\ldots,\boldsymbol{e}_{K}\}$ and generated samples with a hat, e.g. $\hat{\boldsymbol{c}}$, to distinguish them from true data samples. 
The (closed) simplex is defined $\Delta^{D} := \{ \boldsymbol{x} \in \mathbb{R}^{K} : x_i \ge 0,\ \sum_{i=1}^{K} x_i = 1 \}$,  and the open simplex as $\mathring{\Delta}^D := \{ \boldsymbol{x} \in \Delta^{D} : x_i > 0\ \forall i \}$. The boundary $\partial \Delta^{D} := \Delta^{D} \setminus \mathring{\Delta}^D$ consists  of all borders where at least one coordinate is zero. We will also make reference to the unit sphere, defined as $\mathbb{S}^{D} := \{ \boldsymbol{x} \in \mathbb{R}^{K} : \sum_{k=1}^{K} x_k^2 = 1 \}$, and the positive orthant of the sphere $\mathbb{S}^D_{+} := \{ \boldsymbol{x} \in \mathbb{S}^D : x_i \ge 0 \ \forall i \}$. We use the notation $\arg\max \boldsymbol{x} := \boldsymbol{e}_j$ with $j = \arg\max_i x_i$ to denote the one-hot vector at the maximizing index.

Equipping $\mathring{\Delta}^D$ with the Fisher information metric yields a $D$-dimensional Riemannian manifold. Intuitively, a Riemannian metric specifies an inner product on the tangent space at each point, which in turn induces geodesics (shortest paths) and distances (see e.g.\ \citet{Lee2018}). For the Fisher metric on the simplex, both geodesics and geodesic distances admit closed-form expressions. 
Similarly, the unit sphere $\mathbb{S}^{D}$ is a $D$-dimensional Riemannian manifold endowed with the canonical metric induced by its embedding in $\mathbb{R}^{K}$.

\subsection{Conditional Flow Matching} \label{sec:cfm}
We use Conditional Flow Matching (CFM) \citep{Lipman2023, Albergo2023}
as the example (Euclidean) generative model, and hence briefly review it here. 

CFM  is a continuous-time normalizing flow for continuous data, where a vector field $\boldsymbol{u}_t(\boldsymbol{x}):\mathbb{R}^D\times [0,1]\to \mathbb{R}^D$  defines the ODE  $\frac{d \boldsymbol{x}_t}{dt} = \boldsymbol{u}_t(\boldsymbol{x}_t)$, with marginals $\boldsymbol{x}_t\sim p_t$ related to $\boldsymbol{u}_t$ by the continuity equation (Eq.~26 in \citealt{Lipman2023}). The vector field transports samples from a simple base distribution $p_0$ to the target distribution $p_1=p_{\mathrm{data}}$. Since the true velocity field $\boldsymbol{u}_t$ is not available, training instead regresses a parametric model $\boldsymbol{v}^\theta_t(\boldsymbol{x}_t)$ toward the conditional velocity $\boldsymbol{u}_t(\boldsymbol{x}_t \mid \boldsymbol{x}_0,\boldsymbol{x}_1)$ along paths interpolating between $p_0$ and $p_1$, which leads to the CFM objective
\begin{equation}
\mathcal{L}_{\mathrm{CFM}}(\theta)
= \mathbb{E}_{\substack{t\sim \mathrm{Unif}(0,1) \\ \boldsymbol{x}_0,\boldsymbol{x}_1\sim\pi(\boldsymbol{x}_0,\boldsymbol{x}_1)}}
  \left\| \boldsymbol{v}^\theta_t\big(\boldsymbol{x}_t\big) - \boldsymbol{u}_t(\boldsymbol{x}_t|\boldsymbol{x}_0,\boldsymbol{x}_1) \right\|_2^2. \label{eq:cfm_loss}
\end{equation}
There are multiple interpolation schemes \cite{Lipman2024}. A simple and commonly used choice is the linear interpolation,  $\boldsymbol{x}_t = (1-t)\boldsymbol{x}_0 + t \boldsymbol{x}_1$, where $\boldsymbol{x}_0 \sim p_0$ and $\boldsymbol{x}_1 \sim p_1$, resulting in target conditional velocity $\boldsymbol{u}_t(\boldsymbol{x}|\boldsymbol{x}_0,\boldsymbol{x}_1)=\dot{\boldsymbol{x}}_t = \boldsymbol{x}_1 - \boldsymbol{x}_0$. 

Two common choices for the coupling $\pi$ of $p_0$ and $p_1$ are the following. The independent coupling draws $(\boldsymbol{x}_0,\boldsymbol{x}_1)$ independently from $\pi(\boldsymbol{x}_0,\boldsymbol{x}_1) = p_0(\boldsymbol{x}_0)p_1(\boldsymbol{x}_1)$ and the optimal transport coupling (OT) constructs an approximation of the optimal transport plan on each minibatch of samples of $p_0$ and $p_1$ \citep{Tong2024}. Concretely, given a batch of samples from each distribution, a cost matrix is computed using Euclidean distances, and a pairing of the samples (transport plan) is obtained that approximately minimizes the total cost while preserving the empirical marginals. The OT coupling can improve training stability and accelerate inference.

\subsection{Riemannian Flow Matching on the Simplex} \label{sec:riem}
To illustrate how previous methods handle the geometry of the simplex, we briefly outline Statistical Flow Matching (SFM)~\citep{Cheng2024, Davis2024}. The basic idea is to map the simplex to a space that is ``easier'' in some computational sense, and learn the flow there, eventually mapping the continuous samples from the model back to discrete observations with another transformation.

SFM transforms the simplex $\Delta^{D}$ to the positive orthant of the unit sphere $\mathbb{S}_+^{D}$, another $D$-dimensional manifold. The geometry of the simplex equipped with the Fisher Information metric is known in closed form~\citep{Miyamoto2024}, and we have the following isomorphism between the simplex with the Fisher-Rao metric and the sphere with its canonical metric:
\begin{align*}
\varphi: \Delta^{D} \to \mathbb{S}_+^{D}, \quad  &\boldsymbol{x}\mapsto \boldsymbol{z} = \sqrt{\boldsymbol{x}}; \\ 
\varphi^{-1}:  \mathbb{S}_+^{D} \to \Delta^{D}, \quad &\boldsymbol{z} \mapsto \boldsymbol{x} = \boldsymbol{z}^2. 
\end{align*}
The transformation gives the change of volume
\begin{equation}
\label{eq:sfm_cov}
    \det \boldsymbol{G}_{\varphi} = \frac{1}{2^{D}} \prod_{i=1}^{K}\frac{1}{\sqrt{x_i}},
\end{equation}
where $ \boldsymbol{G}_{\varphi} = \boldsymbol{J}^\top_{\varphi}\boldsymbol{J}_{\varphi}$ is the pullback metric of the transformation.
Unlike the Fisher–Rao geometry on the simplex, the spherical geometry (exponential–log maps) is well-defined on the boundary of the positive orthant. However, the change-of-variables volume term (Eq.~\ref{eq:sfm_cov}) becomes singular whenever at least one coordinate is zero, so likelihood evaluation is only possible on the open simplex. To address this, SFM uses a lower bound for the categorical likelihood \citep[Eq.~(14)]{Cheng2024}.
Even though working on the sphere avoids some boundary issues, training still requires a Riemannian variant of Flow Matching \citep{Chen2024} with associated geometric machinery. The training objective is
\begin{equation*}
\mathbb{E}_{\substack{t\sim \mathrm{Unif}(0,1) \\ \boldsymbol{x}_0,\boldsymbol{x}_1\sim\pi(\boldsymbol{x}_0,\boldsymbol{x}_1)} }
\left\|\boldsymbol{v}_t^\theta(\boldsymbol{x}_t){-}\mathrm{Log}_{\boldsymbol{x}_t}(\boldsymbol{x}_1)/(1-t)\right\|^2_{g},
\end{equation*}
where $\|\cdot\|g$ is the Riemannian norm, and  $\boldsymbol{x}_t = \mathrm{Exp}_{\boldsymbol{x}_0}( t \boldsymbol{v})$ and $\boldsymbol{v} = \mathrm{Log}_{\boldsymbol{x}_0}(\boldsymbol{x}_1)$ are the exponential and logarithmic maps.

\section{METHOD}
We build on the same idea as SFM, of transforming the simplex to a space that is easier to work on.
Instead of mapping it to another space that still requires Riemannian machinery, we consider transformations that take us to Euclidean space, where all computations are easy and we can directly leverage standard continuous models.
A full bijection from the closed simplex to $\mathbb{R}^D$ obviously does not exist, but our innovation is to construct such a mapping on $\mathring \Delta^D$, the interior of the simplex.
The method operates entirely within  $\mathring \Delta^D$, with information-preserving mappings between the $\mathring \Delta^D$ and the discrete categories. 

The method, coined \emph{\ourname} (\ourmethod), has two main components: (i) A bijection from the open simplex to Euclidean space, for which we provide two concrete alternatives, and (ii) Dirichlet interpolation for handling discrete observations on the boundary, both for lifting them into the open simplex and eventually for mapping generated samples back to discrete categories.

\subsection{Simplex-to-Euclidean bijections} \label{sec:bij}

\paragraph{Compositional Data Analysis}
\citet{Atchison1980, Aitchison1981, Aitchison1982} noted that for compositional data (vectors in the open simplex), the logratio between the components provides a principled method for constructing smooth bijections from the open simplex to Euclidean space. Two classical examples, which will serve as building blocks of the transforms we later propose, are the additive logratio (ALR), for $1\leq k\leq D$ 
\begin{align}
    \mathrm{alr}(\boldsymbol{x})_k :=  \log\tfrac{x_k}{x_{K}}, \ \mathrm{alr}^{-1}(\boldsymbol{z}) =  \mathrm{softmax}([\boldsymbol{z},0]),
    \label{eq:alr}
\end{align}
and the multiplicative logratio (MLR)
\begin{align*}
    \mathrm{mlr}(\boldsymbol{x})_k :=  \log\tfrac{x_k}{1 - \sum_{i=1}^{k} x_{i}}, \
    \mathrm{mlr}^{-1}(\boldsymbol{z})_k =  \frac{e^{z_k}}{\prod_{i=1}^k \big(1+e^{z_i}\big)}.
\end{align*}
The last entry depends on the others, with $x_{K} = 1 - \sum_{i=1}^{D} x_i$.
Note that both  transformations depend on the ordering of the components of $\boldsymbol{x}$, best seen by the explicit reference to an arbitrary (last) component $x_{K}$ in $\mathrm{alr}(\boldsymbol{x})$. Permuting the inputs produces a different map and consequently also a different geometry.

These transformations induce a geometry on $\mathring \Delta^D$ that differs from the Fisher–Rao metric. The Aitchison geometry equips the open simplex with the inner product \citep{Aitchison1983}
\begin{equation*}
\langle \boldsymbol{x}, \boldsymbol{y} \rangle_A := \frac{1}{2K} \sum_{i,j=1}^{K} 
\log\frac{x_i}{x_j}\log\frac{y_i}{y_j},
\end{equation*}
which captures the relative structure of compositions. Addition in this geometry is defined through a perturbation, $\boldsymbol{x} \oplus \boldsymbol{y} = C(x_1 y_1, \dots, x_K y_K)$, where $C(\boldsymbol{z}) = \boldsymbol{z}/\sum_k z_k$ ensures the result remains in the simplex. Distances and inner products are invariant under such perturbation, 
highlighting that the information is carried by the ratios. These properties make the Aitchison geometry particularly well-suited for data in $\mathring \Delta^D$.

The naive ALR and MLR mappings could in principle be used as such (see the Supplement~\ref{app:alr_mlr}), and recently \citet{Chereddy2025} indeed used ALR for modeling discrete data (for CAD generation) with diffusion models, without accounting for the order invariance in any way. We will next introduce two concrete mappings; one is invariant to the order, resolving the fundamental limitation, whereas the other improves on MLR by aligning it to the center point on each space. 

\paragraph{Isometric logratio transform (ILR)} 
We consider the isometric logratio transform \citep{Egozcue2003} 
\begin{equation} \label{eq:ilr}
    \begin{aligned}
    \varphi: \mathring\Delta^{D} \to \mathbb{R}^{D},\quad  &\boldsymbol{x}\mapsto \boldsymbol{z}= \boldsymbol{H} \log \boldsymbol{x}, \\
    \varphi^{-1}:  \mathbb{R}^{D} \to \mathring\Delta^{D}, \quad &\boldsymbol{z}\mapsto \boldsymbol{x} = \mathrm{softmax}( \boldsymbol{H}^\top \boldsymbol{z} ),
    \end{aligned}
\end{equation}
where $\boldsymbol{H}\in \mathbb{R}^{D\times K}$ a Helmert matrix \citep{Lancaster1965} and the rows of $\boldsymbol{H}$ form an orthonormal basis that spans the tangent space of the simplex $T_{\boldsymbol{x}}\Delta^D=\{\boldsymbol{x}\in \mathbb{R}^{K}: \sum_{k=1}^K x_k=0\}$.
Characterization of the Helmert matrix, proof of $\mathcal{O}(K)$ complexity and other details are provided in the Supplement~\ref{app:ilr}.
The volume change is
\begin{equation*}
    \det \boldsymbol{J}_\varphi = \det(\boldsymbol{H}_{1:D,1:D}) \prod_{i=1}^{K} \frac{1}{x_i}.
\end{equation*}
We use ILR because it is invariant to the category order, always giving the same mapping and hence geometry. Moreover it is an isometry from the simplex with the Aitchison geometry to Euclidean space as stated in Theorem~\ref{thm:isometry}.
A direct consequence is that the paths traced by a Flow Matching model in Euclidean space are geometrically consistent with the Aitchison geometry, following the Aitchison geodesics.
\begin{theorem}[Isometry, \citep{Egozcue2003}]\label{thm:isometry}
Let $\langle\cdot,\cdot\rangle_A$ denote the Aitchison inner product on $\mathring \Delta^{D}$ and $\langle\cdot,\cdot\rangle_2$ the standard inner product on $\mathbb{R}^{D}$. For a Helmert matrix $\boldsymbol H\in\mathbb{R}^{D\times K}$, the ILR map satisfies
\begin{equation*}
    \langle \boldsymbol{x}, \boldsymbol{y} \rangle_A \,=\, \langle \varphi(\boldsymbol{x}), \varphi(\boldsymbol{y}) \rangle_2 \quad \forall \boldsymbol{x},\boldsymbol{y}\in \mathring\Delta^D,
\end{equation*}
and, in particular, the ILR map is an isometry between $(\mathring\Delta^{D},\langle\cdot,\cdot\rangle_A)$ and $(\mathbb{R}^{D},\langle\cdot,\cdot\rangle_2)$. 
\end{theorem}
\paragraph{Stick-breaking transform (SB)} 
We additionally consider one order-dependent transformation that improves over the multiplicative logratio by the composition  with a shift which centers the transformation \citep{Carpenter2017}.
For $1\leq k\leq D$, it is defined as:
\begin{equation} \label{eq:sb}
\begin{aligned}
\varphi &: \mathring\Delta^{D} \to \mathbb{R}^{D}, 
\quad \boldsymbol{x} \mapsto \boldsymbol{z}, \\
z_k &= \mathrm{mlr}(\boldsymbol{x})_k 
      - \log\!\left(\frac{1}{K-k}\right), \\
\varphi^{-1} &: \mathbb{R}^{D} \to \mathring\Delta^{D}, 
\quad \boldsymbol{z} \mapsto \boldsymbol{x}, \\
\boldsymbol{x} &= \mathrm{mlr}^{-1}(\boldsymbol{y}), 
\ y_k = z_k + \log\!\left(\frac{1}{K-k}\right).
\end{aligned}
\end{equation}
The last entry is $x_{K} = 1 - \sum_{k=1}^{D} x_k$ . The term $\tfrac{1}{{K}-k}$ centers the transformation such that the zero vector in $\mathbb{R}^{D}$ maps to the vector $[\tfrac{1}{K},..,\tfrac{1}{K}]\in \Delta^D$. For example a Gaussian with zero mean will correspond to a distribution centered at the midpoint of the simplex. 
SB has simple Jacobian determinant,
$
    \det \boldsymbol{J}_\varphi  = \prod_{i=1}^{K}\frac{1}{x_i},
$
and is widely used in probabilistic modeling \citep{Linderman2015, Carpenter2017}.

\subsection{Handling discrete data }

Categorical observations $\boldsymbol{c}\in\{ \boldsymbol{e}_1,\ldots,\boldsymbol{e}_{K}\}$ lie on the boundary $\partial\Delta^D$, requiring two additional tools: a way to map discrete observations into $\mathring\Delta^D$ for training, and a way to map generated continuous samples back to discrete observations at inference. 

We do both using a stochastic interpolation scheme
\[
 {\boldsymbol{x}} = \lambda \boldsymbol{c} + (1-\lambda) \boldsymbol{\varepsilon},
\]
with 
$\boldsymbol{\varepsilon} \sim \mathrm{Dir}(\boldsymbol{\alpha})$. It associates each one-hot vector $\boldsymbol{c}$ with a continuous representation ${\boldsymbol{x}}$ on the simplex. The scheme is inspired by two distinct previous works. On one hand, it generalizes the deterministic interpolation of \citet{Chereddy2025} that 
represents $\boldsymbol{c}$ with $\boldsymbol{x}$ that is pulled from the corner of the simplex slightly towards the centroid.
On the other hand, we leverage a theoretical result of \citet{Stark2024} to prove that we can exactly recover the discrete observations from the continuous relaxation.

\paragraph{Mapping observations to continuous space.}  
During training, we sample ${\boldsymbol{\varepsilon}}$ for each categorical observation at every iteration. This transforms the discrete data into a Dirichlet-interpolated mixture $q_\lambda(\boldsymbol{x})$ characterized formally in Proposition~\ref{prop:cat_prob_tv}, and the continuous generative model is used to approximate this density. The proposition effectively states that modeling the mixture distribution accurately recovers the true categorical distribution in the total variation sense.

\begin{proposition}\label{prop:cat_prob_tv}
\textbf{Categorical probabilities bound:}
Let $\lambda \ge \tfrac{1}{2}$ so that the Dirichlet–interpolated mixture 
$q_\lambda(\boldsymbol{x})=\sum_{k=1}^{K} p_k\, q_\lambda(\boldsymbol{x}\mid \boldsymbol{e}_k)$ (see Eq.~\ref{eq:mixlogp})
has a.s. disjoint component supports contained in the strict $\arg\max$ regions
$\mathcal{R}_k := \{ \boldsymbol{x}\in \mathring \Delta^{D} : x_k > x_j\ \forall j\neq k\}.$
For any density $\tilde q$ on $\mathring \Delta^D$ define the induced (generated) categorical probabilities
$\hat p_k := \int_{\mathcal{R}_k} \tilde q(\boldsymbol{x})\,d\boldsymbol{x}.$
Then the total variation between true and generated categorical laws is bounded by the distance between the distributions:
\begin{equation*}
\mathrm{TV}(\boldsymbol{p},\hat{ \boldsymbol{p}})= \tfrac{1}{2}\sum_{k=1}^{K} | \hat p_k - p_k |
\;\le\; \tfrac{1}{2}\,\|\tilde q - q_\lambda\|_{1}.
    \end{equation*}
In particular, $\|\tilde q - q_\lambda\|_{1}\to 0$ implies $\hat{\boldsymbol{p}} \to \boldsymbol{p}$ in total variation, and the $\arg\max$ discretization of $\tilde q$ recovers exactly the true categorical distribution.
\end{proposition}

\paragraph{Mapping continuous samples to discrete observations.}
Samples drawn from the continuous generative model can be transformed back to discrete categories by selecting the category with the largest entry.
Proposition~\ref{prop:interpolation}, generalized from \citet{Stark2024}, shows that given an interpolated point $\boldsymbol{x}$ for $\lambda > \tfrac{1}{2}$ the operator $\arg\max \boldsymbol{ x}$ recovers the original category $\boldsymbol{c}$ exactly.

\begin{proposition} \label{prop:interpolation}
\textbf{Dirichlet Interpolation:}
Let $\boldsymbol{c} = \boldsymbol{e}_k$ for some $k\in\{1,..,K\}$. Let ${\boldsymbol{x}} := \lambda \boldsymbol{c} + (1-\lambda) \boldsymbol{\varepsilon}$
where $\boldsymbol{\varepsilon} \sim \mathrm{Dir}(\boldsymbol{\alpha})$ with $\alpha_i > 0$.
If $\lambda > \tfrac{1}{2}$, then $\arg\max {\boldsymbol{x}} = \boldsymbol{c}$. For $\lambda = \tfrac{1}{2}$, this holds almost surely under the distribution of $\boldsymbol{\varepsilon}$.
\end{proposition}

\paragraph{Effect of the parameters.}
The interpolation scheme has two parameters, 
 the interpolation constant $\lambda$ and
the Dirichlet concentration $\boldsymbol{\alpha}$.  We next explain how they can be chosen without resorting to hyperparameter optimization.

Figure~\ref{fig:lambda} illustrates the induced mixture for different $\lambda$; for large $\lambda$ the probability mass remains close to the vertices where the geometry is more curved geometry. We  characterize the density as a continuous transformation of a unimodal base distribution at the center of the space (e.g. $\boldsymbol{z}_0\sim \mathcal{N}(\boldsymbol{0}, \boldsymbol{I}_D)$), and moving more of the mass towards the centroid likely makes this slightly easier. Consequently, we recommend using the smallest valid choice (Proposition~\ref{prop:interpolation}).
That is, we use $\lambda = \tfrac{1}{2}$.

In absence of prior information, the symmetric Dirichlet $\mathrm{Dir}(\alpha,..,\alpha)$ with
$\alpha \in (0,\infty)$ is the only reasonable choice, with the marginal variance
$
   \Var(x_k)  = \frac{D}{K^2(\alpha K + 1)}.
$
For $\alpha<1$ the mass concentrates near the boundary,
whereas for $\alpha>1$ it concentrates at the interior. Following the above reasoning, we prefer having  the mass away from the border, implying $\alpha \gg 1$.

One way to reason about the optimal value is to investigate the variance more closely. For constant $\alpha$, it decreases with $K$, suggesting a possible dependency with the dimensionality. However, under the ILR transform, the variance in the Euclidean coordinates is constant and independent of $K$, as formulated in Proposition~\ref{prop:euc_cov}. That is, even though the simplex marginals become tighter with $K$, the variance of the transformed variables $\boldsymbol{z} = \varphi(\boldsymbol{x})$ remains constant. The choice hence does not dramatically influence the difficulty of fitting the continuous model, and we settle with a uniform choice of sufficiently large $\alpha=100$. 
For SB we do not have a similar theoretical result, but nevertheless make the same choice.

\begin{proposition} \label{prop:euc_cov}
\textbf{Euclidean covariance:}
Let $\boldsymbol{x}\sim\mathrm{Dir}(\alpha,..,\alpha)$ with $\alpha>0$. Let $\boldsymbol{H}\in\mathbb{R}^{{D}\times K}$ be a Helmert matrix, and $\varphi$ the ILR transform. Then the covariance of $\boldsymbol{z} = \varphi(\boldsymbol{x})$ is
\begin{equation*}
\operatorname{Cov}(\boldsymbol{z})=\psi'(\alpha)\,\boldsymbol{I}_{D},  
\end{equation*}
where $\psi'$ is the trigamma function. In particular, the covariance of $\boldsymbol{z}$  is independent of $K$.
\end{proposition}

\begin{figure}
    \centering    
    \begin{subfigure}{0.32\linewidth}
        \centering
        \includegraphics[width=\linewidth]{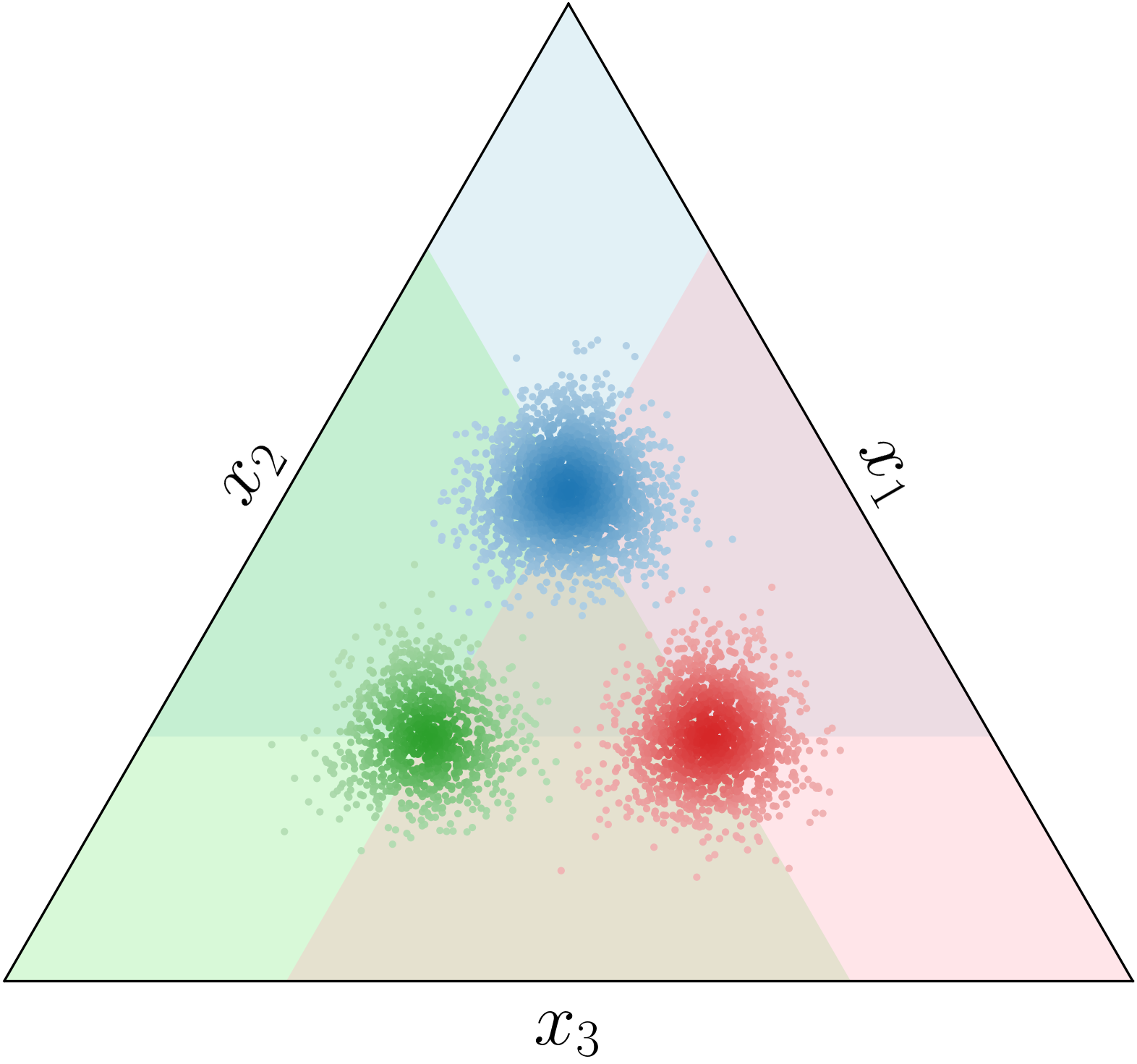}
        \caption*{$\lambda=\tfrac{1}{4}$}
    \end{subfigure}
    \begin{subfigure}{0.32\linewidth}
        \centering
        \includegraphics[width=\linewidth]{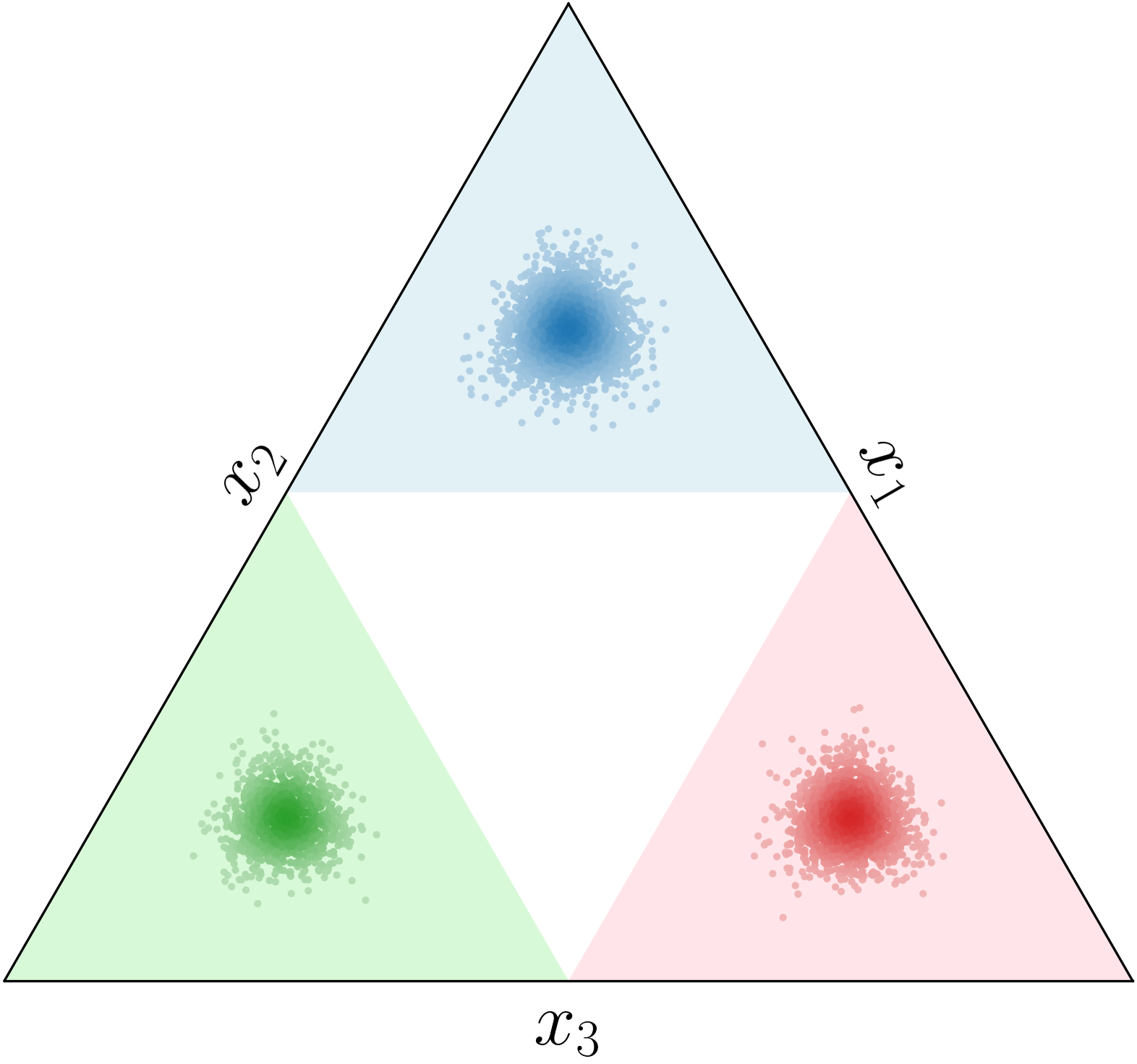}
        \caption*{$\lambda=\tfrac{1}{2}$}
    \end{subfigure}
    \begin{subfigure}{0.33\linewidth}
        \centering
        \includegraphics[width=\linewidth]{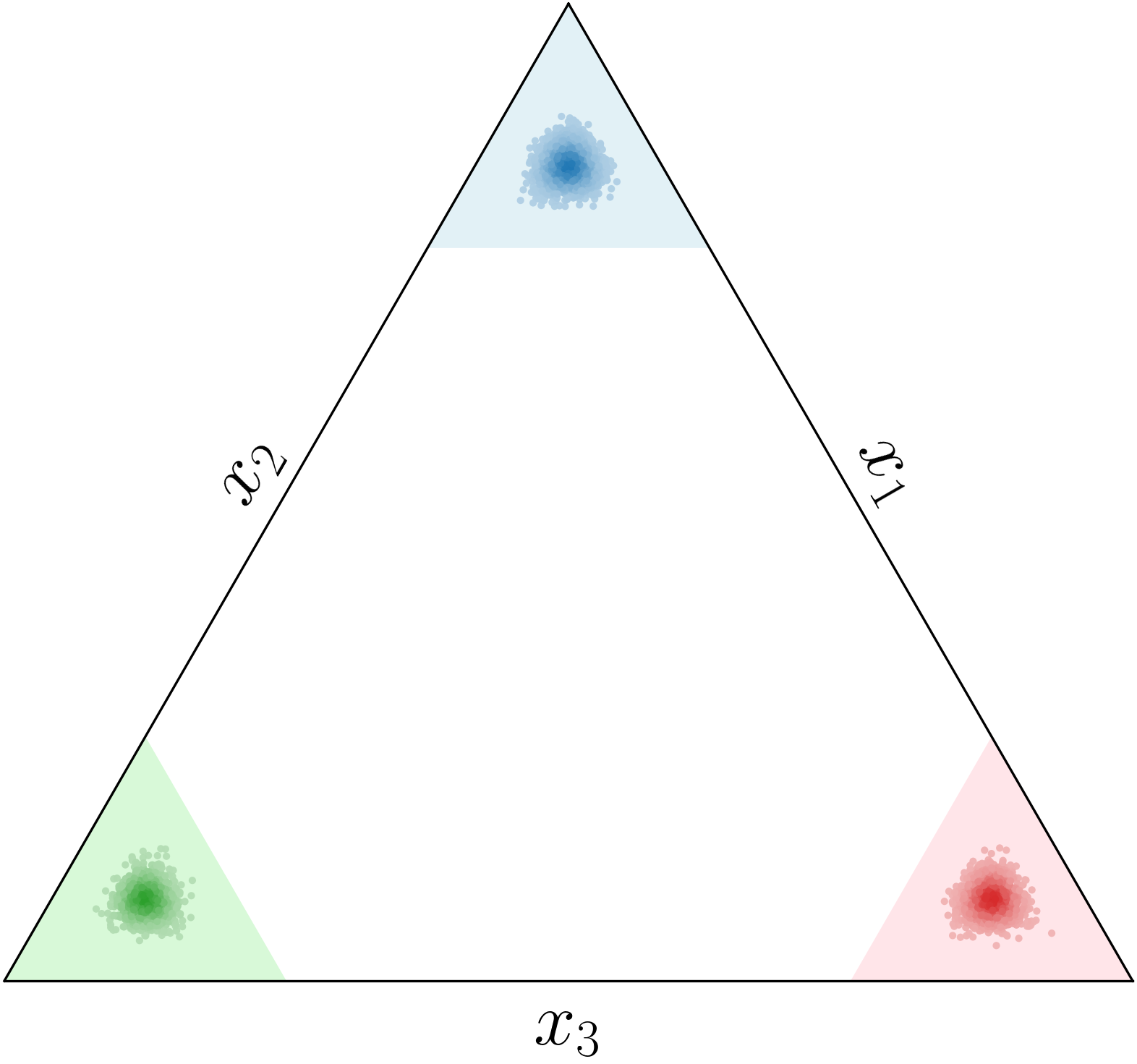}
        \caption*{$\lambda=\tfrac{3}{4}$}
    \end{subfigure}
    \caption{Dirichlet interpolation. The $\lambda$ parameter controls the mixture distribution. For $\lambda \ge \tfrac{1}{2}$ the supports do not overlap and we can recover the categories. Large $\lambda$ unnecessarily concentrates the mass around the simplex borders we want to avoid.  
}
    \label{fig:lambda}
\end{figure}

\begin{figure*}
    \centering
    \begin{subfigure}[b]{0.24\linewidth}
        \centering
        \includegraphics[width=\linewidth]{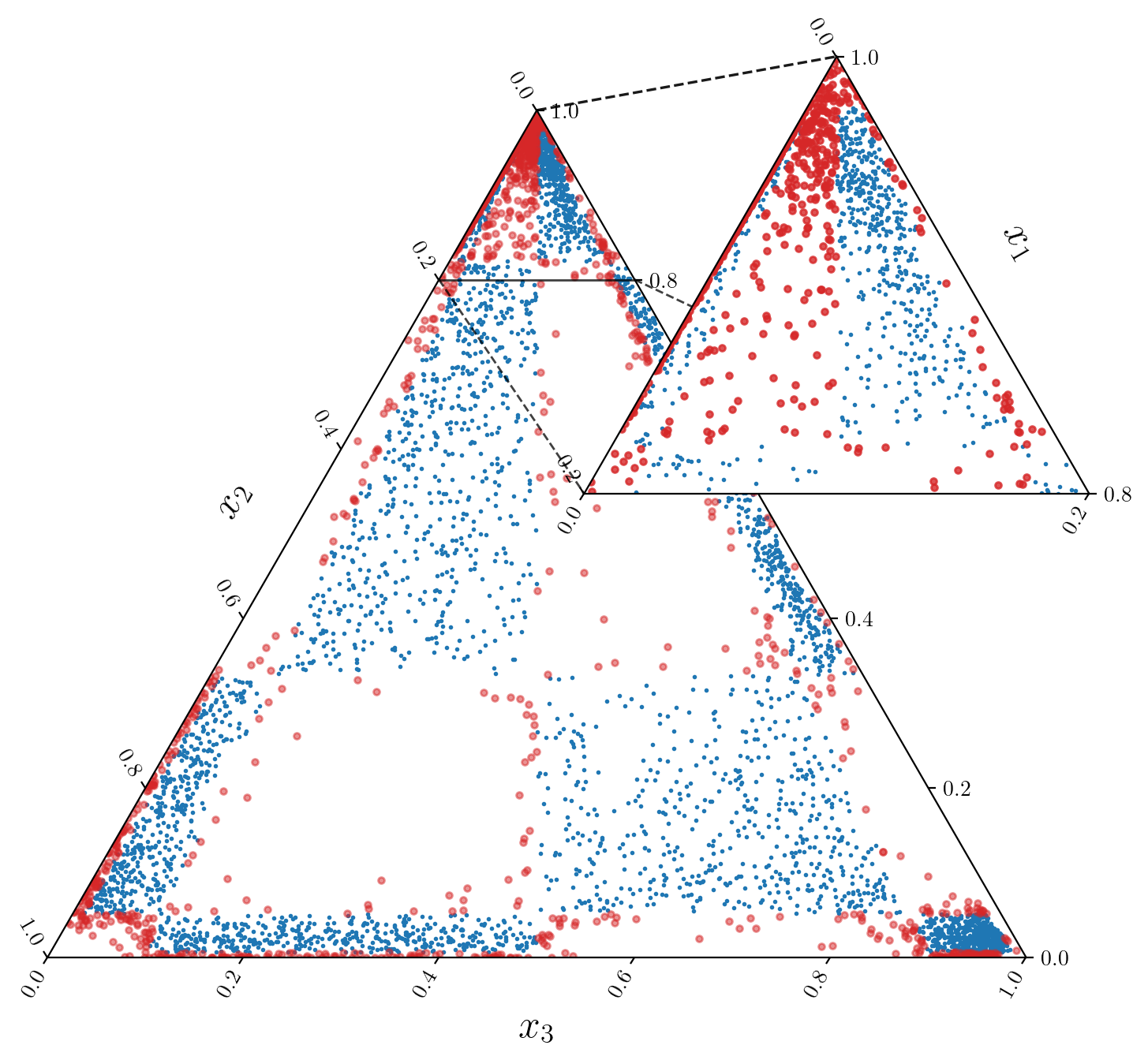}
        \caption{LinearFM, $25.9\%$}
    \end{subfigure}
    \begin{subfigure}[b]{0.24\linewidth}
        \centering
        \includegraphics[width=\linewidth]{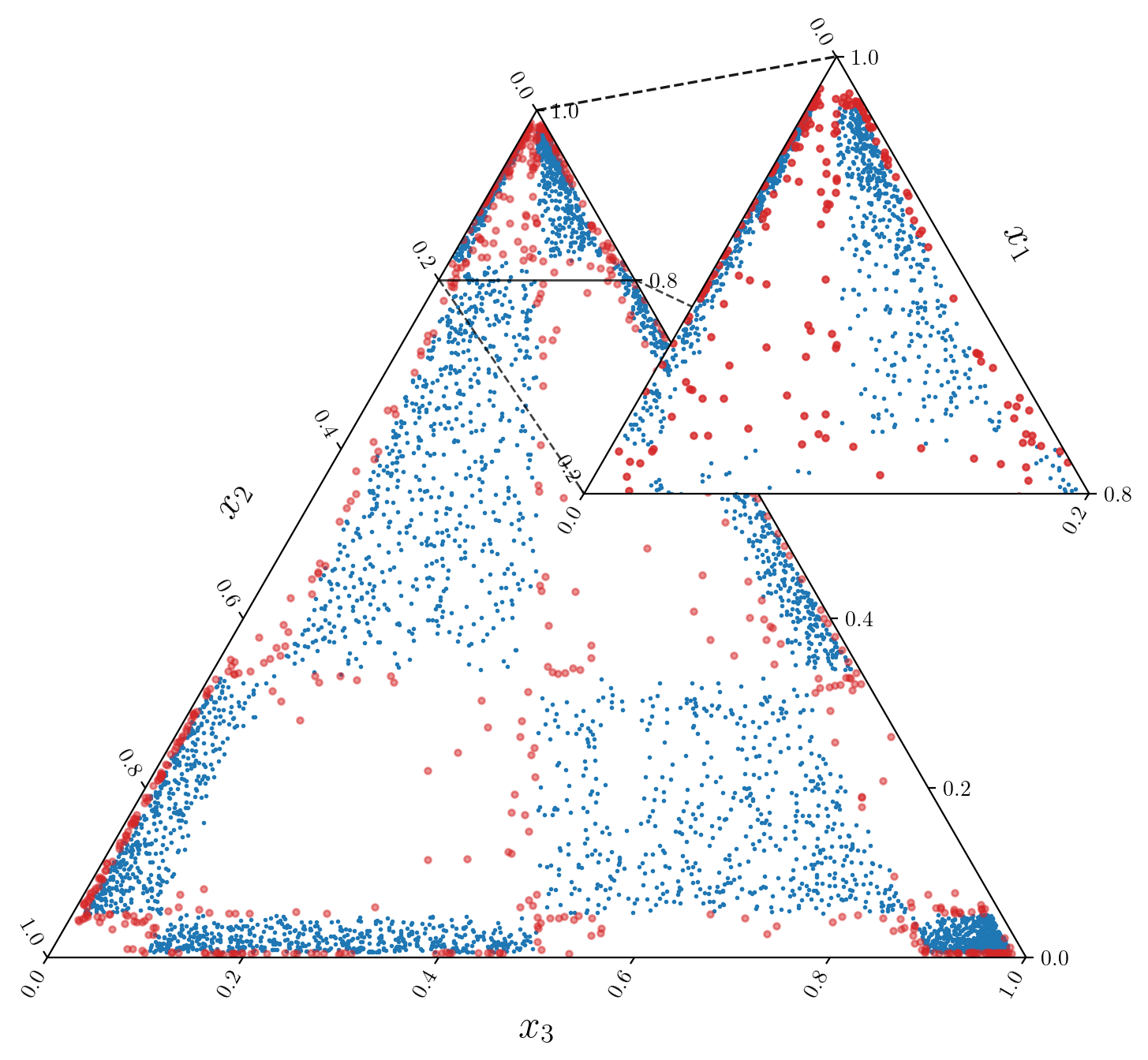}
        \caption{SFM, $12.9\%$}
    \end{subfigure}
    \begin{subfigure}[b]{0.24\linewidth}
        \centering
        \includegraphics[width=\linewidth]{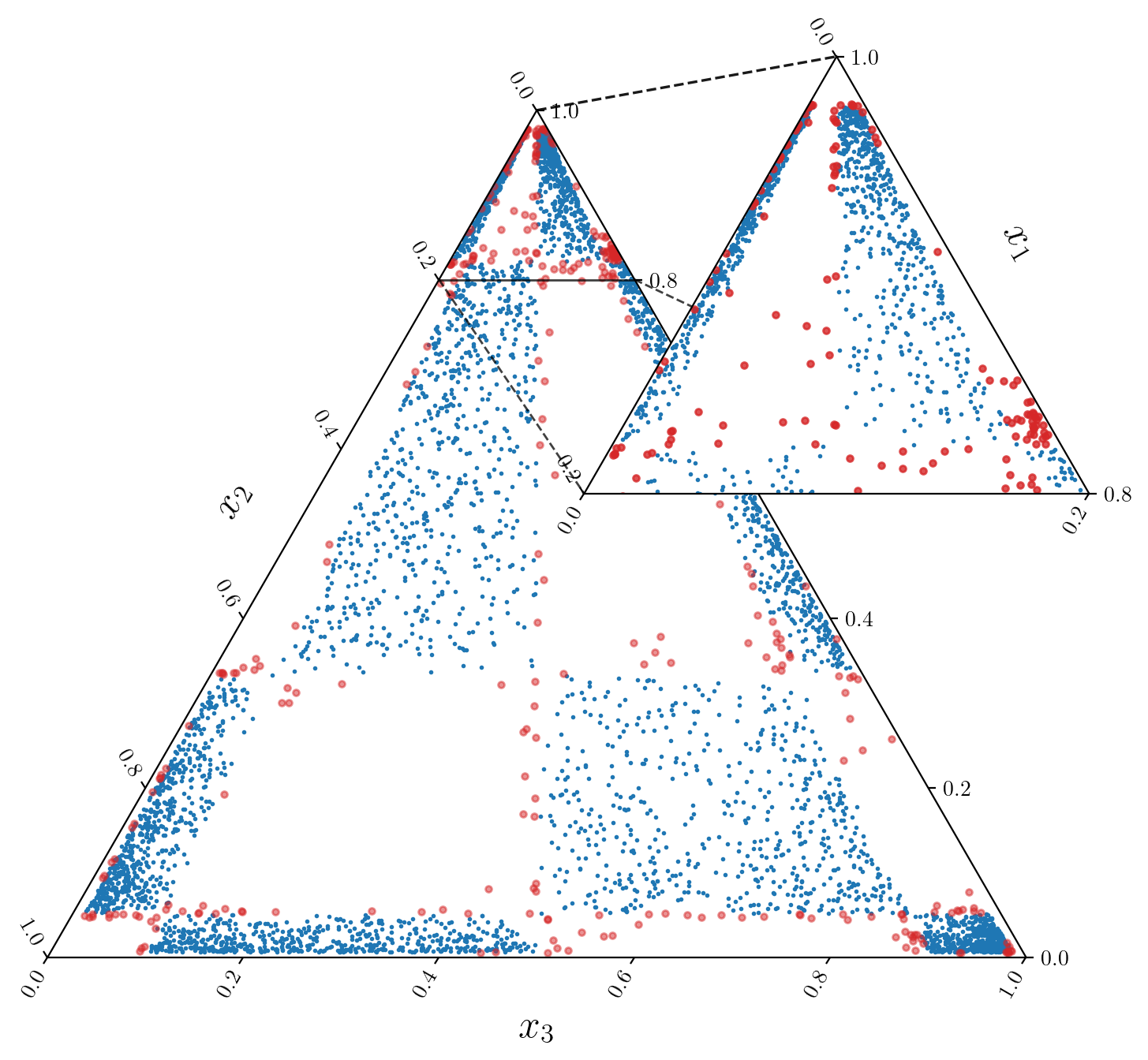}
        \caption{\ourmethod\ (ILR), $6.8 \%$}
    \end{subfigure}
    \begin{subfigure}[b]{0.24\linewidth}
        \centering
        \includegraphics[width=\linewidth]{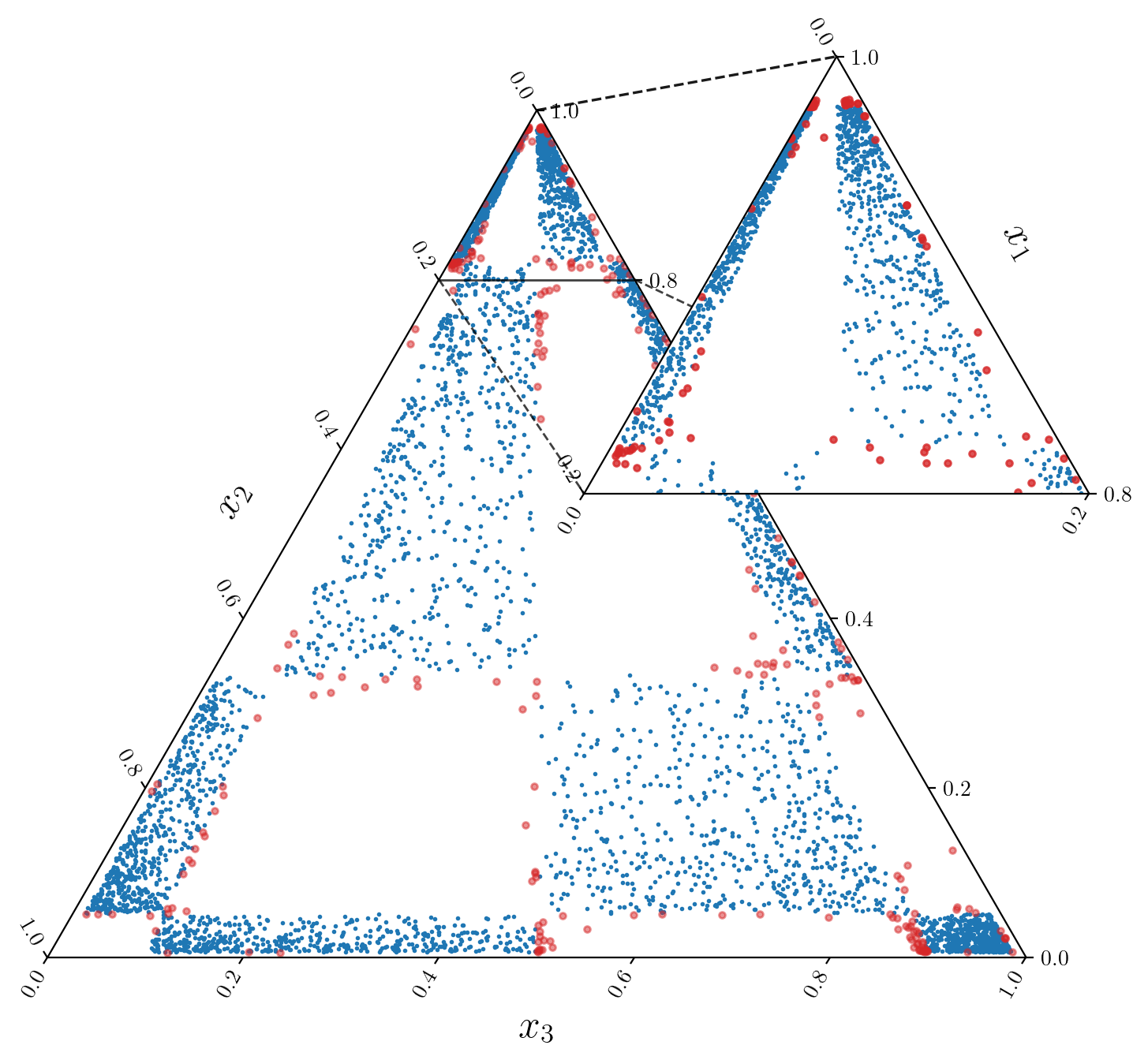}
        \caption{\ourmethod\ (SB), $5.4 \%$}
    \end{subfigure}
    \caption{Samples from Checkerboard on the simplex. Red points indicate samples not aligned with the true density ($\%$ indicated in caption). The zoomed area shows the top region $x_1\geq \tfrac{4}{5}$, emphasizing the differences. }
    \label{fig:checkboard}
\end{figure*}

\subsection{Training, sampling and evaluation} 
The full method is characterized in Algorithms~\ref{alg:train} and \ref{alg:sample}. 
The method can be used in two ways: When the data are discrete categorical observations, we apply Dirichlet interpolation during training and recover categories at sampling via the $\arg\max$ operator. For compositional data (e.g. Fig.~\ref{fig:checkboard}) the interpolation step is not needed.
In both cases, sampling from the underlying continuous model requires solving an ODE as usual; we use standard numerical solvers such as Euler or Dormand–Prince.
\begin{algorithm}
\caption{Training of \ourname\ (\ourmethod)}
\label{alg:train}
\begin{algorithmic}[1]
\Require Data $\boldsymbol{c}\!\in\!\Delta^{D}$, weight $\lambda\!\in\!(0,1)$, 
Dirichlet $\alpha>0$, bijection $\varphi$, base $p_0$ (e.g. $\mathcal{N}(\boldsymbol{0},\boldsymbol{I})$), 
coupling $\pi$ (indep. or minibatch OT), isDiscrete
\For{each mini-batch}
  \For{each $\boldsymbol{c}$ in batch}
    \If{isDiscrete}
      \State Sample $\boldsymbol{\varepsilon}\!\sim\!\mathrm{Dir}(\alpha,..,\alpha)$
      \State ${\boldsymbol{x}} \gets \lambda \boldsymbol{c}+(1-\lambda)\boldsymbol{\varepsilon}$
    \Else
      \State ${\boldsymbol{x}} \gets \boldsymbol{c}$
    \EndIf
  \EndFor
  \State $\boldsymbol{z}_1 \gets \varphi({\boldsymbol{x}})$ \Comment{To Euclidean space}
  \State Sample $\boldsymbol{z}_0 \sim p_0$; pair $(\boldsymbol{z}_0,\boldsymbol{z}_1)$ via $\pi$
  \State Sample $t\sim \mathrm{Unif}(0,1)$
  \State $\boldsymbol{z}_t\!\gets\!(1-t)\boldsymbol{z}_0+t \boldsymbol{z}_1$, \;
         $\boldsymbol{u}_t\!\gets\!\boldsymbol{z}_1-\boldsymbol{z}_0$
  \State Update $\theta$ by 
  $\min \|\boldsymbol{v}_\theta(\boldsymbol{z}_t,t)-\boldsymbol{u}_t\|^2$ \; (Eq.~\eqref{eq:cfm_loss})
\EndFor
\end{algorithmic}
\end{algorithm}
\begin{algorithm}[t]
\caption{Sampling with \ourmethod}
\label{alg:sample}
\begin{algorithmic}[1]
\Require Learned $\boldsymbol{v}_\theta$, bijection $\varphi$, base $p_0$, isDiscrete
\State Sample $\boldsymbol{z}_0 \sim p_0$ 
\State Solve ODE: 
$\tfrac{d\boldsymbol{z}_t}{dt} = \boldsymbol{v}_\theta(\boldsymbol{z}_t,t), \; t\!\in[0,1]$
\State $\hat{\boldsymbol{x}} \gets \varphi^{-1}(\boldsymbol{z}_1)$
\If{isDiscrete}
  \State $\hat{\boldsymbol{c}} = \arg\max \hat{\boldsymbol{x}}$ 
  \Comment{Discrete sample}
\EndIf
\end{algorithmic}
\end{algorithm}

We cannot directly evaluate the categorical probability $\Pr(C=k)$, since the model only defines a continuous density on the interior. We do not need it for learning 
(training uses the interpolated points) or sampling (obtained by the $\arg\max$-operation) 
and it is not provided by many alternatives either. For example, \citet{Cheng2024} only provides an overly loose lower bound (see Supplement~\ref{app:cat_probs} for a demonstration). 
Nonetheless, for evaluation it is useful to estimate $\Pr(C=k)$

Proposition~\ref{prop:interpolation} implies that for any  $\alpha>1$ the true density of the interpolated data $q_{\mathrm{true}}(\boldsymbol{x})$ is a mixture of $K$ Dirichlets, with the modes at the expected values $\boldsymbol{\mu}^{(k)}:= \lambda \boldsymbol{e}_k + (1-\lambda)\tfrac{1}{K}$. Evaluating at $\boldsymbol{\mu}^{(k)}$, we have
$\Pr(C{=}k) = \tfrac{q_{\mathrm{true}}(\boldsymbol{\mu}^{(k)})}{q_{\lambda}(\boldsymbol{\mu}^{(k)}\mid \boldsymbol{e}_k)}$, which naturally leads to the estimator
\begin{equation}\label{eq:catlogp}
    \widehat{\Pr}(C=k) = \frac{q_\theta(\boldsymbol{\mu}^{(k)})}{q_\lambda(\boldsymbol{\mu}^{(k)}|\boldsymbol{e}_k)}.
\end{equation}
The model's density is computed by combining the change of variables given by $\varphi$ and the instantaneous change of variables given by the flow model,
\begin{align*}
    \log q_\theta(\boldsymbol{x}) &= \log p_{0}(\boldsymbol{z}_0) -  \int_0^1 \mathrm{div}(\boldsymbol{v}^\theta_s) (\boldsymbol{z}_s)\dd s + \log \left|\pdv{\boldsymbol{z}_1}{\boldsymbol{x}}\right|.
\end{align*}
Each mixture component is supported on the region of the simplex where its category is the largest, 
$\{\boldsymbol{x}\in \mathring \Delta^D : x_k > x_j\ \forall j\neq k\}$, 
and is obtained by shifting and rescaling a Dirichlet distribution:
\begin{equation}
q_\lambda(\boldsymbol{x}\mid \boldsymbol{e}_k) =\frac{1}{(1-\lambda)^{D}}\mathrm{Dir}\left(\frac{\boldsymbol{x}-\lambda \boldsymbol{e}_k}{1-\lambda}; \alpha\right). \label{eq:mixlogp}
\end{equation}

\section{EXPERIMENTS}
We evaluate our approach on five tasks.
We compare against other continuous models  DirichletFM \citep{Stark2024}, DDSM \citep{Avdeyev2023}, Bit-Diffusion 
\citep{Chen2023}, Gumbel-Softmax FM \citep{Tang2025} and $\alpha$-Flow \citep{Cheng2025}.
As additional baselines, we report results with some methods working directly in the discrete space, 
DFM \citep{Gat2024}, D3PM \citep{Austin2021}, MDLM \citep{Sahoo2024}, SEDD \citep{Lou2024} and MultiFlow \citep{Campbell2024}, to measure the gap between the two families.

We also report results for a baseline using Euclidean geometry directly on the simplex \citep{Chen2024,Stark2024}, denoting it by LinearFM. 
As explained in Section~\ref{sec:cfm}, our method is used with  minibatch OT (w/OT) and without OT, we report both results. We fix $\lambda=\tfrac{1}{2}$ and $\alpha=100$ in all cases. Our code implementation is available at 
\href{https://github.com/williwilliams3/simplexfm}{\texttt{github.com/williwilliams3/simplexfm}}.
\paragraph{Compositional data}
The training data consists of continuous samples in $\mathring \Delta^2$, drawn from a checkerboard distribution in $\mathbb{R}^2$ and projected on the simplex with the inverse stick-breaking transform.  We train directly on these continuous samples rather than on discrete observations. As shown in Fig.~\ref{fig:checkboard}, the generated data from \ourmethod\ aligns more closely with the true distribution, whereas LinearFM and SFM produce many poor samples near the vertices.
The number of invalid samples (points falling in regions of zero density) is more than twice for both, compared to our method.

\begin{table}[t]
\centering
\caption{For NLL, $\le$ is an upper bound and $\approx$ means estimate using Eq.~\ref{eq:catlogp}.
The results above the dotted line are from \citet{Cheng2024}.}
\label{tab:bmnist}
\begin{tabular}{clcc}
\cline{2-4}
 & \textbf{Model} & \textbf{NLL$\downarrow$} & \textbf{FID$\downarrow$} \\
\cline{2-4}
\multirow{2}{*}{\rotatebox[origin=c]{90}{\textbf{Disc.}}} 
& D3PM   & $\leq 0.141 \pm 0.021$ & 67.36 \\
 & DFM    & $0.101 \pm 0.017$ & 34.42 \\  
\cline{2-4}
\multirow{7}{*}{\rotatebox[origin=c]{90}{{\textbf{Continuous}}}}
 & DirichletFM & NA & 77.35 \\
 & DDSM   & $\leq 0.100 \pm 0.001$ & 7.79 \\
 & LinearFM    & NA & 5.91 \\
 & SFM w/ OT   & NA & 4.62 \\
 \cdashline{2-4}
 & \ourmethod(SB)          & $\approx 0.0341 \pm 0.0006$ & 4.93 \\
 & \ourmethod(SB)w/OT    & $\approx 0.0732 \pm 0.0017$ & 4.51 \\
 & \ourmethod(ILR)         & $\approx 0.0851 \pm 0.0051$ & \textbf{4.36} \\
 & \ourmethod(ILR)w/OT   & $\approx 0.0620 \pm 0.0012$ & 4.57 \\
\cline{2-4}
\end{tabular}
\end{table}

\paragraph{Binarized MNIST}
The Binarized MNIST sets each pixel of MNIST to $1$ with probability given by its intensity and $0$ otherwise \citep{Salakhutdinov2008}; each pixel takes value in $\Delta^1$. 
The velocity field is modeled using a convolutional neural network \citep{Song2020, Cheng2024}. We use the standard train/validation/test split and report both negative log-likelihood (NLL)
and the Fréchet inception distance (FID) for the test samples in Table~\ref{tab:bmnist}. 
\ourmethod\ has the lowest NLL and FID, with relatively similar performance for all four variants. Some of the baselines (D3PM, DFM and DirichletFM) are substantially worse.

\paragraph{DNA sequence generation}
We use the human Promoter DNA sequence data from \citet{Avdeyev2023}, with 100,000 sequences of 1024 elements with a transcription signal.
The task is conditional (on the signal) generation of the four symbols in $\Delta^3$.
The training, validation, and test sets are split based on chromosomes: Chromosome 10 is used for validation, Chromosomes 8 and 9 for testing, and the remaining chromosomes for training.
The velocity field is modeled with the same architecture as in \citet{Avdeyev2023}, and following their setup, we evaluate generations by mapping both generated and test samples through a pretrained Sei model \citep{Chen2022}; the SP-MSE loss is the average Euclidean distance between these embeddings. Table~\ref{tab:promoter} shows our method is again the best.
\begin{table}
\centering
\caption{DNA sequence generation. The results for the baselines are from the respective papers.
}
\label{tab:promoter}
    \begin{tabular}{l c}
    \toprule
    \textbf{Model} & \textbf{SP-MSE$\downarrow$} \\
    \midrule
    DDSM & 0.0334 \\
    D3PM-uniform & 0.0375 \\
    Bit-Diffusion (one-hot) & 0.0395 \\
    Bit-Diffusion (bit) & 0.0414 \\    
    Gumbel-Softmax FM &  0.0290\\
    DirichletFM & 0.0269 \\
    LinearFM & 0.0282 \\    
    SFM & 0.0258 \\
    \hdashline
    \ourmethod(SB) &  0.0278\\
    \ourmethod(SB)w/OT &  \textbf{0.0214}\\
    \ourmethod(ILR) &  0.0259\\
    \ourmethod(ILR)w/OT &  0.0224\\
    \bottomrule
\end{tabular}
\end{table}

\paragraph{Text8}
We use the Text8 dataset \citep{Mahoney2011}, which models individual letters as elements of $\Delta^{26}$. It contains 27 symbols (26 letters plus a blank space) and follows the standard 95K/5K/5K split, with each sample being a random chunk of length 256. 
The velocity field is modeled through a 12-layer diffusion transformer \citep{Lou2024,Cheng2024}, and the hyperparameters are selected based on the lowest validation error. As done by \citet{Campbell2024, Cheng2025}, the evaluation is done generating 4000 samples which are tokenized according to the vocabulary of the GPT-J-6B  model \citep{Wang2021}, from the tokens the entropy is obtained and the GPT-J-6B model computes the reported NLL. We desire an entropy close to that of the data distribution and the NLL to be low. Table~\ref{tab:text8} shows that the discrete-space models are the best in terms of NLL, but our method is the best within the continuous relaxations. In terms of entropy, all methods behave fairly similarly.

\begin{table}[t]
\centering
\caption{Text8: NLL and Entropy on the GPT-J-6B model. The baseline results are from the respective papers.}
\label{tab:text8}
\begin{tabular}{cllccc}
\cline{2-5}
 & \textbf{Model} & \textbf{NLL $\downarrow$} & \textbf{Entropy} & \textbf{Diff} \\
\cline{2-5}
\multirow{6}{*}{\rotatebox[origin=c]{90}{{\textbf{Discrete}}}}
 & MDLM            & 6.76 & 7.55 & +0.07 \\
 & DFM             & 6.78 & 7.58 & +0.10 \\
 & D3PM            & 6.93 & 7.38 & -0.10 \\
 & SEDD   & \textbf{6.49} & 7.17 & -0.31 \\
 & MultiFlow       & 6.73 & 7.39 & -0.09 \\
\cline{2-5}
\multirow{8}{*}{\rotatebox[origin=c]{90}{{\textbf{Continuous}}}}
 & LinearFM          & 7.35 & 7.62 & +0.14 \\ 
 & $\alpha$-Flow($\alpha{=}{-}0.5)$  & 7.14 & 7.37 & -0.11 \\
 & $\alpha$-Flow($\alpha{=}0.5$)     & 7.00 & 7.42 & -0.06 \\
 & $\alpha$-Flow($\alpha{=}1$)       & 7.09 & 7.51 & \textbf{+0.03} \\
 & SFM             & 6.85 & 7.38 & -0.10 \\ 
 \cdashline{2-5}
 & \ourmethod(ILR)         & 6.81 & 7.39 & -0.09 \\
 & \ourmethod(ILR)w{/}OT   & 6.89 & 7.38 & -0.10 \\
\cline{2-5}
 & Data & 4.10 & 7.48 & 0 \\
\cline{2-5}
\end{tabular}
\end{table}

\paragraph{Scalability}
The previous experiments all considered problems of fixed dimensionality. To study the performance as a function of $K$, we consider a setup where $10^6$ samples are drawn from discrete distributions of $K=2^1,\dots,2^9$ categories, with some probability vector $\boldsymbol{p} \in \Delta^{D}$.
After training the model, we draw $10^5$ samples and compute the
KL divergence between the empirical density of the samples and the true distribution.
Fig.~\ref{fig:corners} shows our method generally outperforms SFM and LinearFM especially for medium dimensionalities, and is comparable to the discrete-state SEDD for dimensionality up to $K=2^7$. The Supplement~\ref{app:effect_params} has additional experiments under the same setup on the effect of $\alpha$ and $\lambda$, which were kept fixed in all of the main experiments.

\begin{figure}
    \centering
    \includegraphics[width=0.99\linewidth]{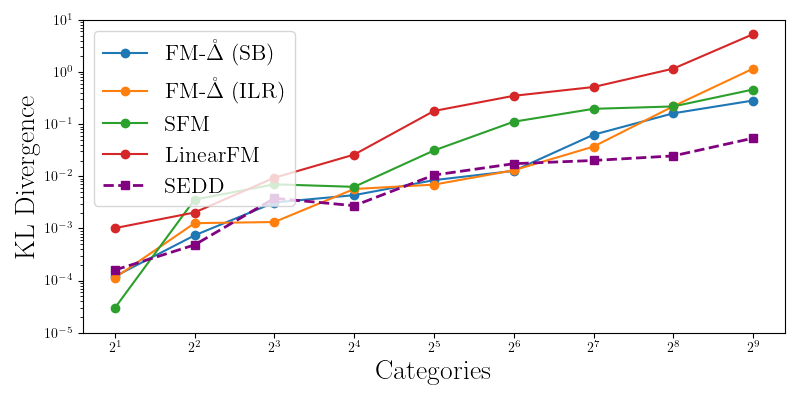}
    \caption{Divergence between the ground truth and estimated categorical probabilities, for problems of varying number of categories.}    
    \label{fig:corners}
\end{figure}

\section{DISCUSSION}

The progress of continuous relaxations for generative modeling of discrete data is extremely rapid. We briefly detail the connections to the most recent works. 

\citet{Diederen2025} used the ILR transformation in modeling compositional data, as one transformation within a composition of multiple ones. They did not, however, consider discrete observations at all. \citet{Potapczynski2020}, in turn, considered an alternative bijection building on $\mathrm{softmax}_{++}$ that also allows recovering discrete data with $\arg\max$, but their solution is designed specifically for a variational autoencoder as the generative model, with no obvious generalization for other models.

\citet{Chereddy2025} proposed a diffusion model that leverages special cases of our machinery. Inspired by logistic-normal models \citep{Atchison1980}, their Gaussian-Softmax diffusion model uses the push-forward of a Gaussian with the additive logratio transform (Eq.~\ref{eq:alr}), using a deterministic interpolation scheme similar to ours ($\lambda \boldsymbol{e}_k + (1-\lambda)\tfrac{1}{K}$ with $\lambda=0.99$) to handle discrete observations. However, their transformation leaves the mass very close to the borders and the solution is specifically geared towards use of the order-dependent ALR transformation.
\citet{Cheng2025} proposed an alternative bijection on the logit-simplex space, as a limiting case of the information geometry, resulting in a geometry that is intrinsically different from ours. Closer inspection of the differences would be highly interesting future work.
\citet{Mahabadi2024} and \citet{Tae2025} also used the logit-simplex space, but treated each discrete token as a point in the probability simplex close to a vertex, rather than making a formal connection with the vertices and well-defined points in the interior.

We specifically focused on accounting for the simplex geometry, but some methods also operate directly in a Euclidean space, by adding random noise 
directly for the one-hot observations in $\mathbb{R}^{K}$. For example, \citet{Hoogeboom2020,Hoogeboom2021} proposed a variational objective for learning a generative model and also used the $\arg\max$ operation like us.
These models miss all connection to the simplex and its natural geometry.
Truncated Flows \citep{Tan2022} extend this idea by learning normalizing flows defined over bounded regions of $\mathbb{R}^{K}$ (or a latent space), recovering the exact categories if the regions are disjoint, and Categorical Normalizing Flows \citep{Lippe2020} model categories using mixtures of logistic distributions, with a particular focus on structured data such as graphs.

Our method aligns with these concurrent works, pushing forward the boundary of how to best use established tools for continuous generation for discrete data, especially in terms of generality.
Our results indicate the method is competitive with the rest while achieving a certain degree of conceptual and computational elegance due to accounting for the simplex geometry without needing complex tools. 

\section{CONCLUSION}
We proposed a principled method that constructs a bridge between the generation of continuous data in Euclidean space and the generation of discrete data in the simplex, enabling use of broad range of Euclidean generative models for categorical data. We demonstrated the approach using flow matching with highly promising results, constraining to this choice to keep the presentational and empirical complexity at bay, but other generative models like \citep{Karras2022, Song2023, Geng2025} could directly be plugged in as future work.

\subsubsection*{Acknowledgements}
The authors were supported by the Research Council of Finland Flagship programme: Finnish Center for Artificial Intelligence FCAI, and additionally by grants: 363317 (BW, AK), 345811 (BW, AK), 348952 (MH), 369502 (MH), 359207 and  the Finland Fellowship granted by Aalto University (VYS).
The authors acknowledge the research environment provided by ELLIS Institute Finland, and then CSC - IT Center for Science, Finland, and the Aalto Science-IT project for computational resources.

\bibliographystyle{plainnat}
\bibliography{bibliography}

\begin{thebibliography}{49}
\providecommand{\natexlab}[1]{#1}
\providecommand{\url}[1]{\texttt{#1}}
\expandafter\ifx\csname urlstyle\endcsname\relax
  \providecommand{\doi}[1]{doi: #1}\else
  \providecommand{\doi}{doi: \begingroup \urlstyle{rm}\Url}\fi

\bibitem[Aitchison(1981)]{Aitchison1981}
John Aitchison.
\newblock A new approach to null correlations of proportions.
\newblock \emph{Journal of the International Association for Mathematical Geology}, 13\penalty0 (2):\penalty0 175--189, 1981.

\bibitem[Aitchison(1982)]{Aitchison1982}
John Aitchison.
\newblock The statistical analysis of compositional data.
\newblock \emph{Journal of the Royal Statistical Society: Series B (Methodological)}, 44\penalty0 (2):\penalty0 139--160, 1982.

\bibitem[Aitchison(1983)]{Aitchison1983}
John Aitchison.
\newblock Principal component analysis of compositional data.
\newblock \emph{Biometrika}, 70\penalty0 (1):\penalty0 57--65, 1983.

\bibitem[Albergo and Vanden-Eijnden(2023)]{Albergo2023}
Michael~S Albergo and Eric Vanden-Eijnden.
\newblock Building normalizing flows with stochastic interpolants.
\newblock In \emph{11th International Conference on Learning Representations, ICLR 2023}, 2023.

\bibitem[Atchison and Shen(1980)]{Atchison1980}
John Atchison and Sheng~M Shen.
\newblock Logistic-normal distributions: Some properties and uses.
\newblock \emph{Biometrika}, 67\penalty0 (2):\penalty0 261--272, 1980.

\bibitem[Austin et~al.(2021)Austin, Johnson, Ho, Tarlow, and Van Den~Berg]{Austin2021}
Jacob Austin, Daniel~D Johnson, Jonathan Ho, Daniel Tarlow, and Rianne Van Den~Berg.
\newblock Structured denoising diffusion models in discrete state-spaces.
\newblock \emph{Advances in neural information processing systems}, 34:\penalty0 17981--17993, 2021.

\bibitem[Avdeyev et~al.(2023)Avdeyev, Shi, Tan, Dudnyk, and Zhou]{Avdeyev2023}
Pavel Avdeyev, Chenlai Shi, Yuhao Tan, Kseniia Dudnyk, and Jian Zhou.
\newblock Dirichlet diffusion score model for biological sequence generation.
\newblock In \emph{International Conference on Machine Learning}, pages 1276--1301. PMLR, 2023.

\bibitem[Campbell et~al.(2024)Campbell, Yim, Barzilay, Rainforth, and Jaakkola]{Campbell2024}
Andrew Campbell, Jason Yim, Regina Barzilay, Tom Rainforth, and Tommi Jaakkola.
\newblock Generative flows on discrete state-spaces: Enabling multimodal flows with applications to protein co-design.
\newblock In \emph{International Conference on Machine Learning}, pages 5453--5512. PMLR, 2024.

\bibitem[Carpenter et~al.(2017)Carpenter, Gelman, Hoffman, Lee, Goodrich, Betancourt, Brubaker, Guo, Li, and Riddell]{Carpenter2017}
Bob Carpenter, Andrew Gelman, Matthew~D Hoffman, Daniel Lee, Ben Goodrich, Michael Betancourt, Marcus Brubaker, Jiqiang Guo, Peter Li, and Allen Riddell.
\newblock Stan: A probabilistic programming language.
\newblock \emph{Journal of statistical software}, 76:\penalty0 1--32, 2017.

\bibitem[Chen et~al.(2022)Chen, Wong, Troyanskaya, and Zhou]{Chen2022}
Kathleen~M Chen, Aaron~K Wong, Olga~G Troyanskaya, and Jian Zhou.
\newblock A sequence-based global map of regulatory activity for deciphering human genetics.
\newblock \emph{Nature genetics}, 54\penalty0 (7):\penalty0 940--949, 2022.

\bibitem[Chen and Lipman(2024)]{Chen2024}
Ricky~TQ Chen and Yaron Lipman.
\newblock Flow matching on general geometries.
\newblock In \emph{The Twelfth International Conference on Learning Representations}, 2024.

\bibitem[Chen et~al.(2023)Chen, ZHANG, and Hinton]{Chen2023}
Ting Chen, Ruixiang ZHANG, and Geoffrey Hinton.
\newblock Analog bits: Generating discrete data using diffusion models with self-conditioning.
\newblock In \emph{The Eleventh International Conference on Learning Representations}, 2023.

\bibitem[Cheng et~al.(2024)Cheng, Li, Peng, and Liu]{Cheng2024}
Chaoran Cheng, Jiahan Li, Jian Peng, and Ge~Liu.
\newblock Categorical flow matching on statistical manifolds.
\newblock \emph{arXiv preprint arXiv:2405.16441}, 2024.

\bibitem[Cheng et~al.(2025)Cheng, Li, Fan, and Liu]{Cheng2025}
Chaoran Cheng, Jiahan Li, Jiajun Fan, and Ge~Liu.
\newblock $\alpha$-flow: A unified framework for continuous-state discrete flow matching models.
\newblock \emph{arXiv preprint arXiv:2504.10283}, 2025.

\bibitem[Chereddy and Femiani(2025)]{Chereddy2025}
Sathvik Chereddy and John Femiani.
\newblock Sketchdnn: Joint continuous-discrete diffusion for cad sketch generation.
\newblock \emph{arXiv preprint arXiv:2507.11579}, 2025.

\bibitem[Davis et~al.(2024)Davis, Kessler, Petrache, Bose, et~al.]{Davis2024}
Oscar Davis, Samuel Kessler, Mircea Petrache, Avishek~Joey Bose, et~al.
\newblock Fisher flow matching for generative modeling over discrete data.
\newblock \emph{arXiv preprint arXiv:2405.14664}, 2024.

\bibitem[Diederen and Zamboni(2025)]{Diederen2025}
Tomek Diederen and Nicola Zamboni.
\newblock Flows on convex polytopes.
\newblock \emph{arXiv preprint arXiv:2503.10232}, 2025.

\bibitem[Dunn and Koes(2024)]{Dunn2024}
Ian Dunn and David~Ryan Koes.
\newblock Mixed continuous and categorical flow matching for 3d de novo molecule generation.
\newblock \emph{CoRR}, 2024.

\bibitem[Egozcue et~al.(2003)Egozcue, Pawlowsky-Glahn, Mateu-Figueras, and Barcelo-Vidal]{Egozcue2003}
Juan~Jos{\'e} Egozcue, Vera Pawlowsky-Glahn, Gl{\`o}ria Mateu-Figueras, and Carles Barcelo-Vidal.
\newblock Isometric logratio transformations for compositional data analysis.
\newblock \emph{Mathematical geology}, 35\penalty0 (3):\penalty0 279--300, 2003.

\bibitem[Eijkelboom et~al.(2024)Eijkelboom, Bartosh, Andersson~Naesseth, Welling, and van~de Meent]{Eijkelboom2024}
Floor Eijkelboom, Grigory Bartosh, Christian Andersson~Naesseth, Max Welling, and Jan-Willem van~de Meent.
\newblock Variational flow matching for graph generation.
\newblock \emph{Advances in Neural Information Processing Systems}, 37:\penalty0 11735--11764, 2024.

\bibitem[Floto et~al.(2023)Floto, Jonsson, Nica, Sanner, and Zhu]{Floto2023}
Griffin Floto, Thorsteinn Jonsson, Mihai Nica, Scott Sanner, and Eric~Zhengyu Zhu.
\newblock Diffusion on the probability simplex.
\newblock \emph{arXiv preprint arXiv:2309.02530}, 2023.

\bibitem[Gat et~al.(2024)Gat, Remez, Shaul, Kreuk, Chen, Synnaeve, Adi, and Lipman]{Gat2024}
Itai Gat, Tal Remez, Neta Shaul, Felix Kreuk, Ricky~TQ Chen, Gabriel Synnaeve, Yossi Adi, and Yaron Lipman.
\newblock Discrete flow matching.
\newblock \emph{Advances in Neural Information Processing Systems}, 37:\penalty0 133345--133385, 2024.

\bibitem[Geng et~al.(2025)Geng, Deng, Bai, Kolter, and He]{Geng2025}
Zhengyang Geng, Mingyang Deng, Xingjian Bai, J~Zico Kolter, and Kaiming He.
\newblock Mean flows for one-step generative modeling.
\newblock \emph{arXiv preprint arXiv:2505.13447}, 2025.

\bibitem[Hoogeboom et~al.(2020)Hoogeboom, Cohen, and Tomczak]{Hoogeboom2020}
Emiel Hoogeboom, Taco~S Cohen, and Jakub~M Tomczak.
\newblock Learning discrete distributions by dequantization.
\newblock \emph{arXiv preprint arXiv:2001.11235}, 2020.

\bibitem[Hoogeboom et~al.(2021)Hoogeboom, Nielsen, Jaini, Forr{\'e}, and Welling]{Hoogeboom2021}
Emiel Hoogeboom, Didrik Nielsen, Priyank Jaini, Patrick Forr{\'e}, and Max Welling.
\newblock Argmax flows and multinomial diffusion: Learning categorical distributions.
\newblock \emph{Advances in neural information processing systems}, 34:\penalty0 12454--12465, 2021.

\bibitem[Karras et~al.(2022)Karras, Aittala, Aila, and Laine]{Karras2022}
Tero Karras, Miika Aittala, Timo Aila, and Samuli Laine.
\newblock Elucidating the design space of diffusion-based generative models.
\newblock \emph{Advances in neural information processing systems}, 35:\penalty0 26565--26577, 2022.

\bibitem[Lancaster(1965)]{Lancaster1965}
HO~Lancaster.
\newblock The helmert matrices.
\newblock \emph{The American Mathematical Monthly}, 72\penalty0 (1):\penalty0 4--12, 1965.

\bibitem[Lee(2018)]{Lee2018}
John~M Lee.
\newblock \emph{Introduction to Riemannian manifolds}, volume~2.
\newblock Springer, 2018.

\bibitem[Linderman et~al.(2015)Linderman, Johnson, and Adams]{Linderman2015}
Scott Linderman, Matthew~J Johnson, and Ryan~P Adams.
\newblock Dependent multinomial models made easy: Stick-breaking with the p{\'o}lya-gamma augmentation.
\newblock \emph{Advances in neural information processing systems}, 28, 2015.

\bibitem[Lipman et~al.(2023)Lipman, Chen, Ben-Hamu, Nickel, and Le]{Lipman2023}
Yaron Lipman, Ricky~TQ Chen, Heli Ben-Hamu, Maximilian Nickel, and Matthew Le.
\newblock Flow matching for generative modeling.
\newblock In \emph{The Eleventh International Conference on Learning Representations}, 2023.

\bibitem[Lipman et~al.(2024)Lipman, Havasi, Holderrieth, Shaul, Le, Karrer, Chen, Lopez-Paz, Ben-Hamu, and Gat]{Lipman2024}
Yaron Lipman, Marton Havasi, Peter Holderrieth, Neta Shaul, Matt Le, Brian Karrer, Ricky~TQ Chen, David Lopez-Paz, Heli Ben-Hamu, and Itai Gat.
\newblock Flow matching guide and code.
\newblock \emph{arXiv preprint arXiv:2412.06264}, 2024.

\bibitem[Lippe and Gavves(2020)]{Lippe2020}
Phillip Lippe and Efstratios Gavves.
\newblock Categorical normalizing flows via continuous transformations.
\newblock \emph{arXiv preprint arXiv:2006.09790}, 2020.

\bibitem[Lou et~al.(2024)Lou, Meng, and Ermon]{Lou2024}
Aaron Lou, Chenlin Meng, and Stefano Ermon.
\newblock Discrete diffusion modeling by estimating the ratios of the data distribution.
\newblock In \emph{Proceedings of the 41st International Conference on Machine Learning}, pages 32819--32848, 2024.

\bibitem[Mahabadi et~al.(2024)Mahabadi, Ivison, Tae, Henderson, Beltagy, Peters, and Cohan]{Mahabadi2024}
Rabeeh~Karimi Mahabadi, Hamish Ivison, Jaesung Tae, James Henderson, Iz~Beltagy, Matthew~E Peters, and Arman Cohan.
\newblock Tess: Text-to-text self-conditioned simplex diffusion.
\newblock In \emph{Proceedings of the 18th Conference of the European Chapter of the Association for Computational Linguistics (Volume 1: Long Papers)}, pages 2347--2361, 2024.

\bibitem[Mahoney(2011)]{Mahoney2011}
Matt Mahoney.
\newblock Large text compression benchmark, 2011.

\bibitem[Miyamoto et~al.(2024)Miyamoto, Meneghetti, Pinele, and Costa]{Miyamoto2024}
Henrique~K Miyamoto, F{\'a}bio~CC Meneghetti, Julianna Pinele, and Sueli~IR Costa.
\newblock On closed-form expressions for the fisher--rao distance.
\newblock \emph{Information Geometry}, 7\penalty0 (2):\penalty0 311--354, 2024.

\bibitem[Potapczynski et~al.(2020)Potapczynski, Loaiza-Ganem, and Cunningham]{Potapczynski2020}
Andres Potapczynski, Gabriel Loaiza-Ganem, and John~P Cunningham.
\newblock Invertible gaussian reparameterization: Revisiting the gumbel-softmax.
\newblock \emph{Advances in Neural Information Processing Systems}, 33:\penalty0 12311--12321, 2020.

\bibitem[Sahoo et~al.(2024)Sahoo, Arriola, Schiff, Gokaslan, Marroquin, Chiu, Rush, and Kuleshov]{Sahoo2024}
Subham Sahoo, Marianne Arriola, Yair Schiff, Aaron Gokaslan, Edgar Marroquin, Justin Chiu, Alexander Rush, and Volodymyr Kuleshov.
\newblock Simple and effective masked diffusion language models.
\newblock \emph{Advances in Neural Information Processing Systems}, 37:\penalty0 130136--130184, 2024.

\bibitem[Sahoo et~al.(2025)Sahoo, Deschenaux, Gokaslan, Wang, Chiu, and Kuleshov]{Sahoo2025}
Subham~Sekhar Sahoo, Justin Deschenaux, Aaron Gokaslan, Guanghan Wang, Justin Chiu, and Volodymyr Kuleshov.
\newblock The diffusion duality.
\newblock \emph{arXiv preprint arXiv:2506.10892}, 2025.

\bibitem[Salakhutdinov and Murray(2008)]{Salakhutdinov2008}
Ruslan Salakhutdinov and Iain Murray.
\newblock On the quantitative analysis of deep belief networks.
\newblock In \emph{Proceedings of the 25th international conference on Machine learning}, pages 872--879, 2008.

\bibitem[Song and Ermon(2020)]{Song2020}
Yang Song and Stefano Ermon.
\newblock Improved techniques for training score-based generative models.
\newblock \emph{Advances in neural information processing systems}, 33:\penalty0 12438--12448, 2020.

\bibitem[Song et~al.(2023)Song, Dhariwal, Chen, and Sutskever]{Song2023}
Yang Song, Prafulla Dhariwal, Mark Chen, and Ilya Sutskever.
\newblock Consistency models.
\newblock In \emph{International Conference on Machine Learning}, pages 32211--32252. PMLR, 2023.

\bibitem[Stark et~al.(2024)Stark, Jing, Wang, Corso, Berger, Barzilay, and Jaakkola]{Stark2024}
Hannes Stark, Bowen Jing, Chenyu Wang, Gabriele Corso, Bonnie Berger, Regina Barzilay, and Tommi Jaakkola.
\newblock Dirichlet flow matching with applications to dna sequence design.
\newblock \emph{arXiv preprint arXiv:2402.05841}, 2024.

\bibitem[Szegedy et~al.(2016)Szegedy, Vanhoucke, Ioffe, Shlens, and Wojna]{Szegedy2016}
Christian Szegedy, Vincent Vanhoucke, Sergey Ioffe, Jon Shlens, and Zbigniew Wojna.
\newblock Rethinking the inception architecture for computer vision.
\newblock In \emph{Proceedings of the IEEE conference on computer vision and pattern recognition}, pages 2818--2826, 2016.

\bibitem[Tae et~al.(2025)Tae, Ivison, Kumar, and Cohan]{Tae2025}
Jaesung Tae, Hamish Ivison, Sachin Kumar, and Arman Cohan.
\newblock Tess 2: A large-scale generalist diffusion language model.
\newblock \emph{arXiv preprint arXiv:2502.13917}, 2025.

\bibitem[Tan et~al.(2022)Tan, Huang, Sordoni, and Courville]{Tan2022}
Shawn Tan, Chin-Wei Huang, Alessandro Sordoni, and Aaron Courville.
\newblock Learning to dequantise with truncated flows.
\newblock In \emph{International conference on learning representations}, 2022.

\bibitem[Tang et~al.(2025)Tang, Zhang, Tong, and Chatterjee]{Tang2025}
Sophia Tang, Yinuo Zhang, Alexander Tong, and Pranam Chatterjee.
\newblock Gumbel-softmax flow matching with straight-through guidance for controllable biological sequence generation.
\newblock \emph{arXiv preprint arXiv:2503.17361}, 2025.

\bibitem[Tong et~al.(2024)Tong, FATRAS, Malkin, Huguet, Zhang, Rector-Brooks, Wolf, and Bengio]{Tong2024}
Alexander Tong, Kilian FATRAS, Nikolay Malkin, Guillaume Huguet, Yanlei Zhang, Jarrid Rector-Brooks, Guy Wolf, and Yoshua Bengio.
\newblock Improving and generalizing flow-based generative models with minibatch optimal transport.
\newblock \emph{Transactions on Machine Learning Research}, 2024.

\bibitem[Wang and Komatsuzaki(2021)]{Wang2021}
Ben Wang and Aran Komatsuzaki.
\newblock Gpt-j-6b: A 6 billion parameter autoregressive language model, 2021.

\end{thebibliography}

\section*{Checklist}



\begin{enumerate}

  \item For all models and algorithms presented, check if you include:
  \begin{enumerate}
    \item A clear description of the mathematical setting, assumptions, algorithm, and/or model. [Yes]
    \item An analysis of the properties and complexity (time, space, sample size) of any algorithm. [Yes]
    \item (Optional) Anonymized source code, with specification of all dependencies, including external libraries. [Yes]
  \end{enumerate}

  \item For any theoretical claim, check if you include:
  \begin{enumerate}
    \item Statements of the full set of assumptions of all theoretical results. [Yes]
    \item Complete proofs of all theoretical results. [Yes]
    \item Clear explanations of any assumptions. [Yes]     
  \end{enumerate}

  \item For all figures and tables that present empirical results, check if you include:
  \begin{enumerate}
    \item The code, data, and instructions needed to reproduce the main experimental results (either in the supplemental material or as a URL). [Yes]
    \item All the training details (e.g., data splits, hyperparameters, how they were chosen). [Yes]
    \item A clear definition of the specific measure or statistics and error bars (e.g., with respect to the random seed after running experiments multiple times). [Yes]
    \item A description of the computing infrastructure used. (e.g., type of GPUs, internal cluster, or cloud provider). [Yes]
  \end{enumerate}

  \item If you are using existing assets (e.g., code, data, models) or curating/releasing new assets, check if you include:
  \begin{enumerate}
    \item Citations of the creator If your work uses existing assets. [Yes]
    \item The license information of the assets, if applicable. [Yes]
    \item New assets either in the supplemental material or as a URL, if applicable. [Not Applicable]
    \item Information about consent from data providers/curators. [Not Applicable]
    \item Discussion of sensible content if applicable, e.g., personally identifiable information or offensive content. [Not Applicable]
  \end{enumerate}

  \item If you used crowdsourcing or conducted research with human subjects, check if you include:
  \begin{enumerate}
    \item The full text of instructions given to participants and screenshots. [Not Applicable]
    \item Descriptions of potential participant risks, with links to Institutional Review Board (IRB) approvals if applicable. [Not Applicable]
    \item The estimated hourly wage paid to participants and the total amount spent on participant compensation. [Not Applicable]
  \end{enumerate}

\end{enumerate}

\clearpage
\appendix
\thispagestyle{empty}

\onecolumn

\setcounter{theorem}{0}
\setcounter{proposition}{0}

\aistatstitle{Simplex-to-Euclidean Bijections for Categorical Flow Matching: \\
Supplementary Materials}

\section{MATHEMATICAL DERIVATIONS}\label{app:math}

\subsection{Proof of proposition \ref{prop:cat_prob_tv_supp}}
\begin{proposition}\label{prop:cat_prob_tv_supp}
\textbf{Categorical probabilities bound:}
Let $\lambda \ge \tfrac{1}{2}$ so that the Dirichlet–interpolated mixture 
$q_\lambda(\boldsymbol{x})=\sum_{k=1}^{K} p_k\, q_\lambda(\boldsymbol{x}\mid \boldsymbol{e}_k)$ 
has a.s. disjoint component supports contained in the strict $\arg\max$ regions
$\mathcal{R}_k := \{ \boldsymbol{x}\in \mathring \Delta^{D} : x_k > x_j\ \forall j\neq k\}.$
For any density $\tilde q$ on $\mathring \Delta^D$ define the induced (generated) categorical probabilities
$\hat p_k := \int_{\mathcal{R}_k} \tilde q(\boldsymbol{x})\,d\boldsymbol{x}.$
Then the total variation between true and generated categorical laws is bounded by the distance between the distributions:
\begin{equation*}
\mathrm{TV}(p,\hat p)= \tfrac{1}{2}\sum_{k=1}^{K} | \hat p_k - p_k |
\;\le\; \tfrac{1}{2}\,\|\tilde q - q_\lambda\|_{1}.
    \end{equation*}
In particular, $\|\tilde q - q_\lambda\|_{1}\to 0$ implies $\hat p \to p$ in total variation, and the $\arg\max$ discretization of $\tilde q$ recovers exactly the true categorical distribution.
\end{proposition}

\begin{proof}
    Because $\lambda \ge \tfrac{1}{2}$ and Prop.~\ref{prop:interpolation_supp} the mixture places almost surely all mass of component $k$ inside $\mathcal{R}_k$, hence
$p_k = \int_{\mathcal{R}_k} q_\lambda(\boldsymbol{x})\,d\boldsymbol{x}$.
For any measurable set $B$, $\big|\int_B (\tilde q(\boldsymbol{x}) - q_\lambda(\boldsymbol{x}))\,d\boldsymbol{x}\big| \le \int_B |\tilde q(\boldsymbol{x}) - q_\lambda(\boldsymbol{x})|\,d\boldsymbol{x}$, giving a point-wise bound. Summing and using $\sum_k \int_{\mathcal{R}_k} | \tilde q(\boldsymbol{x}) - q_\lambda(\boldsymbol{x})|\,d\boldsymbol{x} \le \int | \tilde q(\boldsymbol{x}) - q_\lambda(\boldsymbol{x})|\,d\boldsymbol{x}$ yields 
\begin{equation*}
\tfrac{1}{2}\sum_{k=1}^{K} | \hat p_k - p_k |
\;\le\; \tfrac{1}{2}\,\|\tilde q - q_\lambda\|_{1}.    
\end{equation*}
\end{proof}

\subsection{Proof of proposition \ref{prop:interpolation_supp}}
\begin{proposition} \label{prop:interpolation_supp}
\textbf{Dirichlet Interpolation:}
Let $\boldsymbol{c} = \boldsymbol{e}_k$ for some $k\in\{1,..,K\}$. Let ${\boldsymbol{x}} := \lambda \boldsymbol{c} + (1-\lambda) \boldsymbol{\varepsilon}$
where $\boldsymbol{\varepsilon} \sim \mathrm{Dir}(\boldsymbol{\alpha})$ with $\alpha_i > 0$.
If $\lambda > \tfrac{1}{2}$, then $\arg\max {\boldsymbol{x}} = \boldsymbol{c}$. For $\lambda = \tfrac{1}{2}$, this holds almost surely under the distribution of $\boldsymbol{\varepsilon}$.
\end{proposition}

\begin{proof}
Let $\boldsymbol{c} \in \Delta^D$ be a vector such that $\boldsymbol{c} = \boldsymbol{e}_k$.
The noisy sample is ${\boldsymbol{x}} := \lambda \boldsymbol{c} + (1-\lambda) \boldsymbol{\varepsilon}$ with entries
\begin{equation*}
{x}_k = \lambda + (1-\lambda)\varepsilon_k, \qquad {x}_j = (1-\lambda)\varepsilon_j \quad \text{for } j \neq k.
\end{equation*}
We need to show that ${x}_k > {x}_j$ for all $j \neq k$ when $\lambda > \tfrac{1}{2}$.
This is equivalent to
\begin{equation*}
 \lambda + (1-\lambda)\varepsilon_k > (1-\lambda)\varepsilon_j, \iff \lambda > (1-\lambda)(\varepsilon_j - \varepsilon_k) \iff \frac{\lambda}{1-\lambda} > \varepsilon_j - \varepsilon_k.
\end{equation*}
Since $\varepsilon$ lies in the simplex, $\varepsilon_j - \varepsilon_k \leq 1$, and
for $\lambda > \tfrac{1}{2}$, we have $\frac{\lambda}{1 - \lambda} > 1 \geq \varepsilon_j - \varepsilon_k$.
Consequently, the inequality holds, implying
\begin{equation*}
{x}_k > {x}_j \quad \forall j \neq i,  \ \text{and} \ \arg\max_j {x}_j = k.
\end{equation*}
For the boundary case $\lambda = \tfrac{1}{2}$, the condition becomes
\begin{equation*}
\frac{\lambda}{1-\lambda} = 1, \quad \text{ then } \quad 1 \geq \varepsilon_j - \varepsilon_k \quad (\text{equivalently } {x}_k \geq {x}_j).
\end{equation*}
Note the equality ${x}_k = {x}_j$ occurs iff $\varepsilon_j - \varepsilon_k = 1$, i.e., $\varepsilon_j = 1$ and $\varepsilon_k = 0$.
Under any Dirichlet distribution with positive concentration parameters, this boundary event has probability zero.
Therefore, when $\lambda = \tfrac{1}{2}$ we have ${x}_k > {x}_j$ for all $j \neq k$ almost surely, and thus $\arg\max_j {x}_j = k$ almost surely.
\end{proof}

\subsection{Proof of proposition \ref{prop:euc_cov_supp}}
\begin{proposition} \label{prop:euc_cov_supp}
\textbf{Euclidean covariance:}
Let $\boldsymbol{x}\sim\mathrm{Dir}(\alpha,..,\alpha)$ with $\alpha>0$. Let $\boldsymbol{H}\in\mathbb{R}^{{D}\times K}$ be a Helmert matrix, and $\varphi$ the ILR transform. Then the covariance of $\boldsymbol{z} = \varphi(\boldsymbol{x})$ is
\begin{equation*}
\operatorname{Cov}(\boldsymbol{z})=\psi'(\alpha)\,\boldsymbol{I}_{D},   
\end{equation*}
where $\psi'$ is the trigamma function. In particular, the covariance of $\boldsymbol{z}$  is independent of $K$.
\end{proposition}

\begin{proof}
Recall $D:=K-1$ and $\boldsymbol{x} \sim \operatorname{Dir}(\alpha,..,\alpha)$.
Let $\psi$ denote the digamma function and $\psi' = \frac{d}{d x} \psi(x)$ the trigamma function. The first and second moments of the log-random variable are:
\begin{equation*}
\mathbb{E}[\log x_k]=\psi(\alpha)-\psi(K\alpha),
\qquad
\operatorname{Cov}( \log \boldsymbol{x})=  \psi'(\alpha)\,\boldsymbol{I}_K-\psi'(K\alpha)\,\mathbf{1}\mathbf{1}^\top.
\end{equation*}
Therefore the covariance of $\boldsymbol{z}=\boldsymbol{H}\log \boldsymbol{x}$ is
$$
\operatorname{Cov}(\boldsymbol{z})=\boldsymbol{H}\,\operatorname{Cov}({\log \boldsymbol{x}})\,\boldsymbol{H}^\top
= \psi'(\alpha)\,\boldsymbol{H} \,\boldsymbol{H}^\top - \psi'(K\alpha)\,\boldsymbol{H} \, \mathbf{1}\mathbf{1}^\top \,\boldsymbol{H}^\top.
$$
Use the orthonormality of $\boldsymbol{H}$ and the sum-to-zero property of each row: $\boldsymbol{H} \, \boldsymbol{H}^\top = \boldsymbol{I}_D$, and by construction of the ILR basis $\boldsymbol{H}\mathbf{1}= 0$. Substituting these simplifies the covariance to
$$
\operatorname{Cov}(\boldsymbol{z})=\psi'(\alpha)\,\boldsymbol{I}_D.
$$
Hence, the coordinates $z_k$ are uncorrelated and each has constant variance $\psi'(\alpha)$ independent of $K$, proving the proposition.
\end{proof}

\subsection{Proof of Theorem \ref{thm:isometry_supp}}
\begin{theorem}[Isometry, \citep{Egozcue2003}]\label{thm:isometry_supp}
Let $\langle\cdot,\cdot\rangle_A$ denote the Aitchison inner product on $\mathring \Delta^{D}$ and $\langle\cdot,\cdot\rangle_2$ the standard inner product on $\mathbb{R}^{D}$. For a Helmert matrix $\boldsymbol H\in\mathbb{R}^{D \times K}$, the ILR map satisfies
\begin{equation*}
    \langle \boldsymbol{x}, \boldsymbol{y} \rangle_A \,=\, \langle \varphi(\boldsymbol{x}), \varphi(\boldsymbol{y}) \rangle_2 \quad \forall \boldsymbol{x},\boldsymbol{y}\in \mathring\Delta^D,
\end{equation*}
and, in particular, the ILR map is an isometry between $(\mathring\Delta^{D},\langle\cdot,\cdot\rangle_A)$ and $(\mathbb{R}^{D},\langle\cdot,\cdot\rangle_2)$. 
\end{theorem}

We provide the proof where the Helmert matrix is the orthonormal basis of the tangent space of the simplex $T_{\boldsymbol{x}}\mathring\Delta^D=\{\boldsymbol{v}\in \mathbb{R}^{K}: \sum_{i=1}^K v_i=0\}$. 
Refer to \citet{Egozcue2003} for a proof independent of the choice of basis for $T_{\boldsymbol{x}}\mathring\Delta^D$.
\begin{proof}    
Let $\boldsymbol{x},\boldsymbol{y}\in \mathring \Delta^D$, and set  $\boldsymbol{u}:=\log \boldsymbol{x}$ and $\boldsymbol{v}:=\log \boldsymbol{y}$, the Aitchison inner product is 
\begin{align*}
    \langle \boldsymbol{x},\boldsymbol{y}\rangle_A &= \frac{1}{2K} \sum_{i,j=1}^{K} 
\log\frac{x_i}{x_j}\log\frac{y_i}{y_j} 
    = \frac{1}{2K} \sum_{i,j=1}^{K}(u_i-u_j)(v_i-v_j), \\
 &=\frac{1}{2K} \sum_{i,j=1}^{K} 
 (u_iv_i- u_iv_j -u_jv_i + u_jv_j)= \langle \boldsymbol{u},\boldsymbol{v}\rangle_2 - \tfrac{1}{K} \langle \boldsymbol{u},\mathbf 1\rangle_2 \langle \boldsymbol{v},\mathbf 1\rangle_2.
\end{align*}
On the other hand, the rows of $\boldsymbol{H}$ form a basis of $T_{\boldsymbol{x}}\mathring\Delta^D$, then the orthogonal projection of a vector $\boldsymbol{y} \in \mathbb{R}^K$ onto $T_{\boldsymbol{x}}\mathring{\Delta}^D$ is given by  
$
\mathrm{proj}_{T_{\boldsymbol{x}}\mathring{\Delta}^D}(\boldsymbol{y}) = \boldsymbol{H}^\top \boldsymbol{H}\,\boldsymbol{y}.
$
This projection operator is equivalent to the matrix  
$\boldsymbol{P} = \boldsymbol{I}_K - \frac{1}{K}\mathbf{1}\mathbf{1}^\top$,
so that
$
\mathrm{proj}_{T_{\boldsymbol{x}}\mathring{\Delta}^D}(\boldsymbol{y}) = \left(\boldsymbol{I}_K - \frac{1}{K}\mathbf{1}\mathbf{1}^\top\right)\boldsymbol{y}
$. The norm in the Euclidean space is
\begin{align*}
    \langle\varphi(\boldsymbol{x}), \varphi(\boldsymbol{y})\rangle_2 = \boldsymbol{u}^\top \boldsymbol{H}^\top \boldsymbol{H} \, \boldsymbol{v} = 
    \boldsymbol{u}^\top (\boldsymbol{I}_K-\frac{1}{K}\mathbf 1\mathbf 1^\top) \, \boldsymbol{v} = 
    \langle \boldsymbol{u},\boldsymbol{v}\rangle_2 - \tfrac{1}{K} \langle \boldsymbol{u},\mathbf 1\rangle_2 \langle \boldsymbol{v},\mathbf 1\rangle_2.
\end{align*}
Therefore both the inner products coincide and in particular for all $\boldsymbol{x}\in\mathring\Delta^D$
\begin{equation*}
    \|\boldsymbol{x}\|_A = \|\varphi(\boldsymbol{x})\|_2.
\end{equation*}

\end{proof}

\subsection{Stick-breaking transform} \label{app:unit_simplex}
\citet{Carpenter2017} give an alternative formulation of the SB inverse transform. We prove its equivalence to the SB inverse transform, and use this equivalence to derive a simpler expression of the Jacobian determinant.

The inverse unit-simplex transform (US) is defined for $1\leq k\leq D$
\begin{equation} \label{eq:us}
\begin{aligned}
\varphi^{-1}: \mathbb{R}^{D} \to \Delta^{D}, 
\quad \boldsymbol{z} \mapsto \boldsymbol{x}, 
\quad x_k = \left( 1 - \sum_{i=1}^{k-1} x_{i} \right) \sigma(y_k), \ y_k=   z_k + \log \frac{1}{K - k},
\end{aligned}
\end{equation}
where $\sigma(\cdot)$ denotes the sigmoid function and the last entry is $x_{K}= 1-\sum_{i=1}^D x_i$. Proposition~\ref{prop:sb} shows the equivalence between the stick-breaking and unit-simplex inverse transforms.

\begin{proposition} \label{prop:sb}
The SB inverse transform and the US inverse transform are equal.
\end{proposition}
\begin{proof}
Proof by induction. 
Recall the SB inverse transform is $x_k = \prod_{i=1}^{k-1}(1- \sigma(y_{i})) \sigma(y_k)$. 

For the base case $k=2$, we have $x_1=\sigma(y_{1})$ and  $x_2 = (1-\sigma(y_{1}))\sigma(y_{2})$ in both cases.

Induction step, we assume SB and US coincide for $k$. 

Let us prove  the equality for $k+1$,
\begin{align*}
    x_{k + 1} = &\left( 1 - \sum_{i=1}^{k} x_{i} \right) \sigma(y_{k+1}) =
    \left( 1 - \sum_{i=1}^{k-1} x_{i} -   x_{k} \right) \sigma(y_{k+1}) 
    \;\overset{\substack{1-\sum_{i=1}^{k-1} x_i = \tfrac{x_k}{\sigma(y_{k})}}}{=}
    \left( \frac{x_k}{\sigma(y_{k})} -   x_{k} \right) \sigma(y_{k+1}) \\
       =&\left( \prod_{i=1}^{k-1}(1-\sigma(y_{i})) - \prod_{i=1}^{k-1}(1-\sigma(y_{i}))\sigma(y_{k}) \right) \sigma(y_{k+1})
     = \prod_{i=1}^{k}(1-\sigma(y_{i})) \sigma(y_{k+1}).
\end{align*}
\end{proof}
\paragraph{Computation of the determinant}
Take $y_k := z_k+ \log \left( \frac{1}{K - k}\right)$, then $x_k = \frac{e^{y_k}}{\prod_{i=1}^k (1+e^{y_i})}$
\begin{align*}
    \pdv{x_k}{z_k} &= \frac{e^{y_k}}{\prod_{i=1}^k (1+e^{y_i})}\left(1-\frac{e^{y_k}}{1+e^{y_k}}\right) = \frac{e^{y_k}}{\prod_{i=1}^k (1+e^{y_i})} \left(\frac{1}{1+e^{y_k}}\right)= x_k\left(\frac{1}{1+e^{y_k}}\right)
\end{align*}
Since the Jacobian is lower triangular, the determinant is the product of the diagonal terms
\begin{align*}
    \det \boldsymbol{J}_\varphi & = \prod_{k=1}^{D} x_k\left(\frac{1}{1+e^{y_k}}\right).
\end{align*}
As a side result of Proposition~\ref{prop:sb} we have the equality $1 - \sum_{i=1}^{k}x_i=  \prod_{i=1}^{k}(1-\sigma(y_{i})) $, thus  $ x_{K}=\prod_{i=1}^{D}(1- \sigma(y_i))$, and we obtain 
\begin{align*}
    \det \boldsymbol{J}_\varphi & = \prod_{k=1}^{K} x_k.
\end{align*}

\subsection{Isometric logratio transform} \label{app:ilr}

We state the matrix determinant Lemma~\ref{lemma:mdl} which is useful throughout the derivations.
\begin{lemma} Matrix determinant.  \label{lemma:mdl}
Let $\boldsymbol{A}\in \mathbb{R}^{D\times D}$ be a full rank square matrix and  ${u} \in \mathbb{R}^D$ a vector, then
\begin{equation*}
\det(\boldsymbol{A} + {u}{u}^\top) = \left(1 + {u}^\top \boldsymbol{A}^{-1} {u} \right) \det(\boldsymbol{A}).
\end{equation*}

\end{lemma}

The ILR transform is $\boldsymbol{z}= \boldsymbol{H} \log \boldsymbol{x}$, we fix the last entry as $ x_{K} = 1 - \sum_{i=1}^D x_i $.
The entries of the Jacobian matrix are: 
\begin{equation*}
\frac{\partial z_i}{\partial x_j} =
\frac{h_{ij}}{x_j} - \frac{h_{i,K}}{x_{K}}.
\end{equation*}
The Jacobian can be written in matrix form
\begin{equation*}
\boldsymbol{J}_\varphi = \boldsymbol{H}_{1:D,1:D} \operatorname{diag}\left( \frac{1}{x_1}, \dots, \frac{1}{x_D} \right) - \frac{1}{x_{K}} \boldsymbol{h}_{K}\mathbf{1}^\top,
\end{equation*}
where $\boldsymbol{h}_{K}$ is the last column of $\boldsymbol{H}$ and $\mathbf{1}$ is the vector of ones. 
A property of the Helmert matrix is that the columns of $\boldsymbol{H}$ sum to zero, hence $\boldsymbol{h}_{K} = -\boldsymbol{H}_{1:D,1:D} \mathbf{1}$ and
\begin{equation*}
\boldsymbol{J}_\varphi = \boldsymbol{H}_{1:D,1:D} \left( \operatorname{diag}\left( \frac{1}{x_1}, \dots, \frac{1}{x_D} \right) + \frac{1}{x_{K}}\mathbf{1}\mathbf{1}^\top  \right) .
\end{equation*}

The determinant can be computed with the help of Lemma~\ref{lemma:mdl},
\begin{align*}
    \det(\boldsymbol{J}_\varphi) &= \det(\boldsymbol{H}_{1:D,1:D}) \prod_{i=1}^{D} \tfrac{1}{x_i}
    \left(
    1 + \frac{1}{x_{K} }\sum_{i=1}^{D} x_i
    \right)    \\
    &=
    \det(\boldsymbol{H}_{1:D,1:D}) \prod_{i=1}^{D} \frac{1}{x_i}
    \left(
    1 + \frac{1}{x_{K} } (1- x_{K})
    \right) \\
    &= \det(\boldsymbol{H}_{1:D,1:D}) \prod_{i=1}^{K} \frac{1}{x_i}. \\
\end{align*}
The determinant of the reduced Helmert matrix is $\det(\boldsymbol{H}_{1:D,1:D}) = \frac{1}{\sqrt{K}}$.

\paragraph{Computational complexity}
We derive the linear computational complexity of the ILR bijection.
The (non-full) Helmert matrix $\boldsymbol{H}\in\mathbb{R}^{D\times K}$ has a simple recursive structure.  For $i=1,\dots,D$, the $i$-th row contains $i$ entries equal to
$1/\sqrt{i(i+1)}$, one entry equal to $-\,i/\sqrt{i(i+1)}$ in column $i{+}1$,
and zeros elsewhere:
\begin{equation*}
    \boldsymbol{H} =
\begin{bmatrix}
\frac{1}{\sqrt{1\cdot 2}} & -\frac{1}{\sqrt{1\cdot 2}} & 0 & 0 & \cdots & 0 \\[6pt]
\frac{1}{\sqrt{2\cdot 3}} & \frac{1}{\sqrt{2\cdot 3}} &
-\frac{2}{\sqrt{2\cdot 3}} & 0 & \cdots & 0 \\[6pt]
\frac{1}{\sqrt{3\cdot 4}} & \frac{1}{\sqrt{3\cdot 4}} &
\frac{1}{\sqrt{3\cdot 4}} &
-\frac{3}{\sqrt{3\cdot 4}} & \cdots & 0 \\
\vdots & \vdots & \vdots & \vdots & \ddots & \vdots \\
\frac{1}{\sqrt{D\cdot K}} &
\frac{1}{\sqrt{D\cdot K}} & \cdots &
\frac{1}{\sqrt{D\cdot K}} &
\frac{1}{\sqrt{D\cdot K}} &
-\frac{D}{\sqrt{D\cdot K}}
\end{bmatrix}.
\end{equation*}
Let $\boldsymbol{v}=\log\boldsymbol{x}$.  
Then $\boldsymbol{z} = \boldsymbol{H} \boldsymbol{v}$ has components
\begin{equation*}
z_i
= \frac{1}{\sqrt{i(i+1)}}\!\left(\sum_{j=1}^i v_j - i\,v_{i+1}\right).
    \end{equation*}
Defining the cumulative sum $S_i=\sum_{j=1}^i v_j$ (computed in $\mathcal{O}(K)$), we obtain
\begin{equation*}
z_i = \frac{S_i - i\,v_{i+1}}{\sqrt{i(i+1)}}, \qquad i=1,\dots,D,    
\end{equation*}  
Thus the full matrix–vector product $\boldsymbol{z}= \boldsymbol{H} \boldsymbol{v}$ can be computed in $\mathcal{O}(K)$ time without explicitly forming $\boldsymbol{H}$.

\subsection{Simplex to sphere transform } \label{app:transf_sphere}
Let $\boldsymbol{z}\in \mathbb{S}_+^D$ such that the sphere bijection is $\boldsymbol{z} = \varphi(\boldsymbol{x})=\sqrt{\boldsymbol{x}}$ for $\boldsymbol{x}\in\Delta^D$.
A direct computation gives the Jacobian $\boldsymbol{J}_{\varphi}\in\mathbb{R}^{K\times D}$
\begin{equation*}
\boldsymbol{J}_{\varphi} = \frac{dz}{d x_{1:D}} =
\begin{bmatrix}
\frac{1}{2\sqrt{x_1}} & 0 & \cdots & 0 \\
0 & \frac{1}{2\sqrt{x_2}} & \cdots & 0 \\
\vdots & \vdots & \ddots & \vdots \\
0 & 0 & \cdots & \frac{1}{2\sqrt{x_{D}}} \\
-\frac{1}{2\sqrt{x_{K}}} & -\frac{1}{2\sqrt{x_{K}}} & \cdots & -\frac{1}{2\sqrt{x_{K}}}
\end{bmatrix}_{K \times D}.
\end{equation*}

The pull-back metric is $\boldsymbol{G}_{\varphi} = \boldsymbol{J}_{\varphi}^\top \boldsymbol{J}_{\varphi} \in \mathbb{R}^{D\times D}$;
\begin{align*}
\boldsymbol{G}_{\varphi} = \frac{1}{4}
\left(
\operatorname{diag}\left(\frac{1}{x_1}, \dots, \frac{1}{x_{D}}\right)
+ \frac{1}{x_{K}} \mathbf{1}_{D}\mathbf{1}_{D}^\top
\right),
\end{align*}
This is a rank-1 update of a diagonal matrix,  set $\boldsymbol{A} = \operatorname{diag}\left( \frac{1}{x_1}, \dots, \frac{1}{x_{D}} \right)$, and $u = \sqrt{\frac{1}{x_{K}}} \cdot \mathbf{1}_{D}$, we obtain
\begin{equation*}
1+{u}^\top \boldsymbol{A}^{-1} {u} = \frac{1}{x_{K}} \sum_{i=1}^{D} x_i = 1+\frac{1 - x_{K}}{x_{K}}  = \frac{1 }{x_{K}} .
\end{equation*}
Due to Lemma~\ref{lemma:mdl} the determinant of $\boldsymbol{G}_{\varphi}$ and the volume element are:
\begin{equation*}
\det(\boldsymbol{G}_{\varphi})= \frac{1}{4^{D}}   \prod_{i=1}^{K} \frac{1}{x_i}, \quad 
\sqrt{\det(\boldsymbol{G}_{\varphi})} = \frac{1}{2^{D}} \prod_{i=1}^{K} \frac{1}{\sqrt{x_i}} .
\end{equation*}

\subsection{Categorical probabilities estimation} \label{app:logprob}
We have constructed the estimator of the categorical probabilities 
$$
\widehat{\Pr}(C{=}k) = \tfrac{q_{\theta}(\boldsymbol{\mu}^{(k)})}{q_{\lambda}(\boldsymbol{\mu}^{(k)}\mid \boldsymbol{e}_k)}.$$
For its computation we need two components: the log densities of our model in the simplex $q_\theta(\boldsymbol{x})$ for $\boldsymbol{x} \in \mathring\Delta^{D}$ and the true densities for each mixture component $q_\lambda(\boldsymbol{x}\mid \boldsymbol{e}_k)$.

\paragraph{Computing the distribution of the model in the simplex}
Recall that $\boldsymbol{x}\in \mathring\Delta^{D}$ and $\boldsymbol{z}\in \mathbb{R}^{D}$ and $\boldsymbol{v}^\theta_t$ is the vector field of the flow model. The density in the Euclidean space is given by the instantaneous change of variable formula:
\begin{equation}
    \log p_{1}(\boldsymbol{z}_1) = \log p_{0}(\boldsymbol{z}_0) - \int_0^1 \mathrm{div}(\boldsymbol{v}^\theta_s) (\boldsymbol{z}_s)\,\mathrm{d} s. \label{eq:cnf_logdensity}
\end{equation}
Change of variables for the transformation $\boldsymbol{x} = \varphi^{-1}(\boldsymbol{z}_1)$ gives the density on $\mathring\Delta^D$:
\begin{equation}
    \log q_\theta(\boldsymbol{x}) = \log p_{1}(\varphi(\boldsymbol{x})) + \log \left|\pdv{\varphi(\boldsymbol{x})}{\boldsymbol{x}}\right|.
    \label{eq:cov}
\end{equation}
Combining Equations~\eqref{eq:cnf_logdensity} and~\eqref{eq:cov} we obtain the density over the simplex:
\begin{align*}
    \log q_{\theta}(\boldsymbol{x}) &= \log p_{0}(\boldsymbol{z}_0) -  \int_0^1 \mathrm{div}(\boldsymbol{v}^\theta_s) (\boldsymbol{z}_s)\,\mathrm{d} s + \log \left|\pdv{\boldsymbol{z}_1}{\boldsymbol{x}}\right|.\\
\end{align*}
If $\boldsymbol{z}_0$ is distributed as a standard Gaussian in $\mathbb{R}^D$ then $p_{0}(\boldsymbol{z}_0) = \mathcal{N}(\boldsymbol{z}_0\mid \boldsymbol{0},\boldsymbol{I})$. If the base distribution is the uniform distribution, $\boldsymbol{x}_0\sim \mathrm{Unif}(\Delta^D)$, on the simplex, then $p_{0}$ is computed with an additional change of variables
\begin{equation*}
   \log  p_{0}(\boldsymbol{z}_0) = \log p_{0}(\varphi^{-1}(\boldsymbol{z}_0))  + \log \left|\pdv{\varphi^{-1}(\boldsymbol{z}_0)}{\boldsymbol{z}_0}\right|.
\end{equation*}

\paragraph{Computing the distribution of the mixture components}
The Dirichlet interpolation moves discrete data from the vertices to a Dirichlet mixture. Each mixture component conditioned on  $\boldsymbol{e}_k$ has distribution $q_\lambda(\boldsymbol{x}\mid \boldsymbol{e}_k)$. Let us compute this distribution.
Let $\boldsymbol{\varepsilon}\sim \mathrm{Dir}(\alpha)$ with density $p_{\boldsymbol{\varepsilon}}(\boldsymbol{v})$ and $\alpha>0$. Define the affine map $f:\,\boldsymbol{\varepsilon} \mapsto \boldsymbol{x}=\lambda \boldsymbol{e}_k+(1-\lambda)\boldsymbol{\varepsilon}$. Its inverse is
\begin{equation*}
f^{-1}(\boldsymbol{x})=\frac{\boldsymbol{x}-\lambda \boldsymbol{e}_k}{1-\lambda}.
\end{equation*}
The mapping acts as a scaling by factor $(1-\lambda)$ in $D$ dimensions, hence the Jacobian absolute determinant is
\begin{equation*}
\left|\det \boldsymbol{J}_f^{-1}\right|=
\frac{1}{(1-\lambda)^{D}}.
\end{equation*}
By change of variables,
\begin{equation*}
q_\lambda(\boldsymbol{x}\mid \boldsymbol{e}_k)=p_{\boldsymbol{\varepsilon}}\big(f^{-1}(\boldsymbol{x})\big)\,\big|\det \boldsymbol{J}_f^{-1}\big|.
\end{equation*}
Multiplying by the Jacobian factor $\tfrac{1}{(1-\lambda)^{D}}$ yields  the final density (supported on the truncated simplex $\{\boldsymbol{x} \in \mathring\Delta^D: x_k\ge \lambda\}$)
\begin{equation*}
q_\lambda(\boldsymbol{x}\mid \boldsymbol{e}_k) =\frac{1}{(1-\lambda)^{D}}\mathrm{Dir}(f^{-1}(\boldsymbol{x}); \alpha).
\end{equation*}

\subsection{Standardize the data}
To ensure that the mean compositions of different mixture components are comparable across varying dimensions, we examine how their Aitchison norms scale with the number of categories. The scaling determines how far each component is from the zero vector in $\mathbb{R}^D$. Recall $D = K - 1$ and 
$
\boldsymbol{\mu}^{(k)} = \lambda \boldsymbol{e}_k + (1 - \lambda)\tfrac{1}{K}.
$
The squared Aitchison norm of $\boldsymbol{\mu}^{(k)}$ is
\begin{align}
    \|\boldsymbol{\mu}^{(k)}\|_A^2 
    = \frac{1}{2K}\sum_{i=1}^K\sum_{j=1}^K \left(\log \frac{\mu^{(k)}_i}{\mu^{(k)}_j}\right)^2 
     = \frac{1}{K} \sum_{j \neq k} \left(\log \frac{\mu^{(k)}_k}{\mu^{(k)}_j}\right)^2 
     = \frac{D}{K} \left[\log \left(1 + \frac{K\lambda}{1 - \lambda}\right)\right]^2.
     \label{eq:norm_mean}
\end{align}
For $\lambda = \tfrac{1}{2}$, this becomes
\begin{equation*}
    \|\boldsymbol{\mu}^{(k)}\|_A = \sqrt{\frac{D}{K}} \, \log(K+1).
\end{equation*}
The mean compositions therefore move logarithmically farther from the origin as the number of categories increases, indicating a dimensionality-dependent scaling. Although this effect could be removed by normalizing with $1 / \|\boldsymbol{\mu}^{(k)}\|_A$, our preliminary experiments showed that this adjustment had no notable benefit (see Section~\ref{app:experiments}).

\section{EXPERIMENTAL DETAILS} \label{app:experiments}
The empirical experiments were done in two distinct computing environments for practical reasons, with differing hardware. Hence, we indicate the hardware separately for each experiment.

\paragraph{Compositional Data}
We evaluate the models by drawing $5000$ samples using the Dopri5 solver. 
The velocity field is modeled with a fully connected network consisting of $4$ hidden layers with $512$ units each. 
Training is performed on AMD Rome 7H12 CPUs on a computer cluster.

\paragraph{Binarized MNIST}
We generate samples and estimate likelihoods using the Euler solver with $300$ steps for all four variants of our method. Non–cherry-picked generated examples are shown in Fig.~\ref{fig:bmnist_samples}. 
Following \citet{Cheng2024}, we adopt the CNN architecture of \citet{Song2020}, modified such that each convolutional layer receives a distinct time embedding. 
The negative log-likelihood (NLL) approximation is standardized by evaluating the log-density on both the test image $x$ and its flipped version $1{-}x$, ensuring the total sum is one.  
FID statistics (mean and covariance) are computed with the InceptionV3 model \citep{Szegedy2016} over the entire training dataset, and FID scores are calculated between $1000$ generated samples and the training data. 
We use the same hyperparameters as \citet{Cheng2024}: batch size $256$ and initial learning rate $3\times10^{-4}$. 
Each model is trained for approximately $500$ epochs on a single NVIDIA Volta V100 GPU.

\paragraph{Promoter Design}
Following \citet{Avdeyev2023}, sequences in the training data are randomly offset by up to $10$ positions during training. 
The velocity field is parameterized by the same network as in \citet{Avdeyev2023}, consisting of $20$ stacks of one-dimensional convolutional layers. 
Training is conducted for $200\text{K}$ steps on a single NVIDIA Ampere A100 GPU (40GB) over the hyperparameter grid shown in Table~\ref{tab:promoter_hyperparams}. 
Samples are generated using the Euler solver with $300$ steps for all four method variants. 
Evaluation follows \citet{Avdeyev2023}, using the Sei model with the H3K4me3 chromatin mark to predict active promoters on both generated and test sequences. 
Performance is measured via the squared difference between Sei predictions on both datasets (SP-MSE). 
For each run, model weights are selected based on the CFM validation loss, and the best hyperparameters are chosen with respect to SP-MSE on the validation set.

\paragraph{Text8}
Samples are generated using the Euler solver with $300$ steps for all method variants. 
The velocity field is modeled with a 12-layer diffusion transformer, following \citet{Lou2024,Cheng2024}. 
We conduct a grid search over hyperparameters and select the best model based on validation error. 
Only the ILR bijection is considered, and the hyperparameter grid is:
\[
\text{OT} = [\text{False}, \text{True}], \quad
\text{scaling (Eq.~\ref{eq:norm_mean})} = [\text{True}, \text{False}], \quad
\text{batch size} = [128, 216], \quad
\text{lr} = [10^{-4}, 2\times10^{-4}].
\]
Training is performed in parallel on 4 NVIDIA Ampere A100 GPUs (40GB) for approximately $400$ epochs over 3 days.

\paragraph{Scalability} \label{sec:scalability_supp}
We evaluate scalability by generating $N=10^5$ samples $\{\hat{\boldsymbol{x}}^{(i)}\}_{i=1}^N$ with the Euler solver using $200$ steps and estimating 
$\hat{\boldsymbol{p}} = \tfrac{1}{N}\sum_{i=1}^N \hat{\boldsymbol{x}}^{(i)}$. 
The true distribution $\boldsymbol{p}\in\Delta^{D}$ is defined as $p_1=\tfrac{1}{2}$ and $[p_2,\ldots,p_K]\sim\tfrac{1}{2}\mathrm{Unif}(\Delta^{D})$. 
The velocity field is modeled by a fully connected network with $4$ hidden layers of $512$ units. 
The input dimension is $D{+}64$ for our method and $K{+}64$ for SFM, where $64$ corresponds to the dimension of the sinusoidal time embedding. 
All models share the same fixed hyperparameters. Training is performed on AMD Rome 7H12 CPUs on a computer cluster.

\begin{table}
    \centering
    \caption{Values of SP-MSE on validation and test data for the tested hyperparameters. WD is weight decay and $\beta_1$ is a parameter of the Adam optimizer. }
    \label{tab:promoter_hyperparams}
\begin{tabular}{llrrlllrll}
\toprule
Method & Bijection & OT & Batch & Opt & WD & $\beta_1$ & Step & SP-MSE(val) & SP-MSE(test) \\
\midrule
\ourmethod & SB & False & 64 & Adam & 0 & 0.85 & 40000 & 0.0251 & {0.0278} \\
\ourmethod & SB & False & 64 & Adam & 0 & 0.95 & 40000 & 0.0321 & 0.0341 \\
\ourmethod & SB & False & 128 & Adam & 0 & 0.85 & 30000 & 0.0327 & 0.0353 \\
\ourmethod & SB & False & 128 & Adam & 0 & 0.95 & 30000 & 0.0406 & 0.0443 \\
\ourmethod & SB & False & 128 & Adam & $10^{-5}$ & 0.85 & 120000 & 0.0435 & 0.0447 \\
\ourmethod & SB & False & 64 & Adam & $10^{-5}$ & 0.85 & 200000 & 0.0506 & 0.0512 \\
\ourmethod & SB & False & 128 & Adam & $10^{-5}$ & 0.95 & 120000 & 0.0509 & 0.0514 \\
\ourmethod & SB & False & 64 & Adam & $10^{-5}$ & 0.95 & 200000 & 0.0549 & 0.0554 \\
\hline
\ourmethod & SB & True & 64 & Adam & 0 & 0.85 & 40000 & 0.0213 & {0.0214} \\
\ourmethod & SB & True & 128 & Adam & 0 & 0.85 & 30000 & 0.0314 & 0.0325 \\
\ourmethod & SB & True & 128 & Adam & $10^{-5}$ & 0.95 & 120000 & 0.0363 & 0.038 \\
\ourmethod & SB & True & 128 & Adam & 0 & 0.95 & 30000 & 0.0387 & 0.0409 \\
\ourmethod & SB & True & 128 & Adam & $10^{-5}$ & 0.85 & 120000 & 0.0387 & 0.0392 \\
\ourmethod & SB & True & 64 & Adam & $10^{-5}$ & 0.95 & 190000 & 0.0405 & 0.0435 \\
\ourmethod & SB & True & 64 & Adam & 0 & 0.95 & 40000 & 0.0591 & 0.0569 \\
\ourmethod & SB & True & 64 & Adam & $10^{-5}$ & 0.85 & 190000 & 0.0613 & 0.0642 \\
\midrule
\ourmethod & ILR & False & 64 & Adam & 0 & 0.85 & 40000 & 0.0252 & 0.0259 \\
\ourmethod & ILR & False & 64 & Adam & 0 & 0.95 & 40000 & 0.0321 & 0.0335 \\
\ourmethod & ILR & False & 128 & Adam & 0 & 0.85 & 30000 & 0.0328 & 0.0346 \\
\ourmethod & ILR & False & 128 & Adam & 0 & 0.95 & 30000 & 0.0406 & 0.0443 \\
\ourmethod & ILR & False & 128 & Adam & $10^{-5}$ & 0.85 & 120000 & 0.0436 & 0.0449 \\
\ourmethod & ILR & False & 64 & Adam & $10^{-5}$ & 0.85 & 200000 & 0.0508 & 0.0526 \\
\ourmethod & ILR & False & 128 & Adam & $10^{-5}$ & 0.95 & 120000 & 0.0511 & 0.0511 \\
\ourmethod & ILR & False & 64 & Adam & $10^{-5}$ & 0.95 & 200000 & 0.0548 & 0.0553 \\
\hline
\ourmethod & ILR & True & 64 & Adam & 0 & 0.85 & 40000 & 0.0213 & 0.0224 \\
\ourmethod & ILR & True & 128 & Adam & 0 & 0.85 & 30000 & 0.0314 & 0.0317 \\
\ourmethod & ILR & True & 128 & Adam & $10^{-5}$ & 0.95 & 120000 & 0.0362 & 0.0384 \\
\ourmethod & ILR & True & 128 & Adam & 0 & 0.95 & 30000 & 0.0387 & 0.0419 \\
\ourmethod & ILR & True & 128 & Adam & $10^{-5}$ & 0.85 & 120000 & 0.0387 & 0.039 \\
\ourmethod & ILR & True & 64 & Adam & $10^{-5}$ & 0.95 & 190000 & 0.0405 & 0.042 \\
\ourmethod & ILR & True & 64 & Adam & 0 & 0.95 & 40000 & 0.0588 & 0.0581 \\
\ourmethod & ILR & True & 64 & Adam & $10^{-5}$ & 0.85 & 190000 & 0.0612 & 0.0626 \\
\bottomrule
\end{tabular}
    
\end{table}

\begin{figure}[ht]    
    \centering
    \setlength{\tabcolsep}{5pt} 
    \renewcommand{\arraystretch}{0} 
    \begin{tabular}{cccc}
        \begin{subfigure}{0.22\linewidth}
            \includegraphics[width=\linewidth]{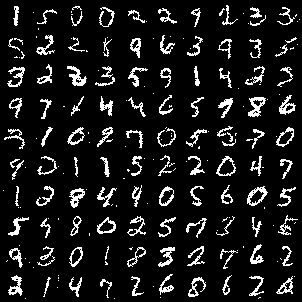}
            \caption{LinearFM}
        \end{subfigure} &
        \begin{subfigure}{0.22\linewidth}
            \includegraphics[width=\linewidth]{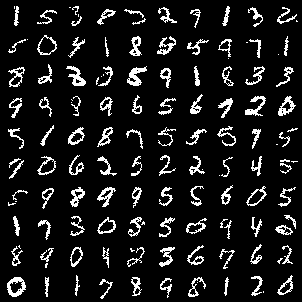}
            \caption{SFM }
        \end{subfigure} &
        \begin{subfigure}{0.22\linewidth}
            \includegraphics[width=\linewidth]{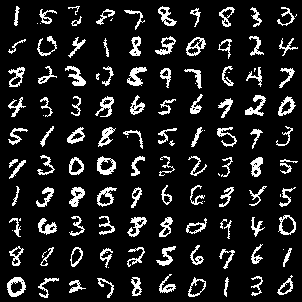}
            \caption{SFM w/ OT}
        \end{subfigure} 
        \\[6pt] 

        \begin{subfigure}{0.22\linewidth}
            \includegraphics[width=\linewidth]{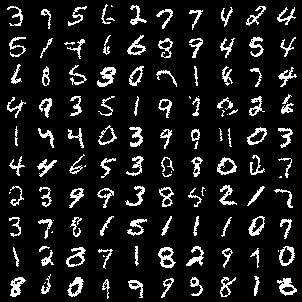}
            \caption{\ourmethod(ILR) }
        \end{subfigure} &
        \begin{subfigure}{0.22\linewidth}
            \includegraphics[width=\linewidth]{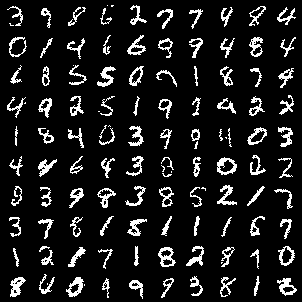}
            \caption{\ourmethod(ILR) w/ OT}
        \end{subfigure} &
        \begin{subfigure}{0.22\linewidth}
            \includegraphics[width=\linewidth]{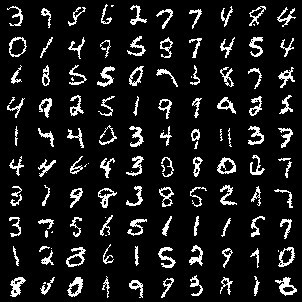}
            \caption{\ourmethod(SB) }
        \end{subfigure} &
        \begin{subfigure}{0.22\linewidth}
            \includegraphics[width=\linewidth]{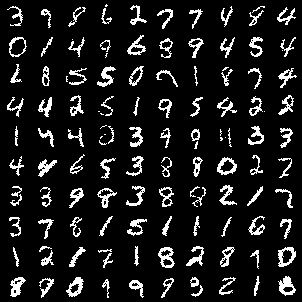}
            \caption{\ourmethod(SB) w/ OT}
        \end{subfigure}
    \end{tabular}
    \caption{Samples from BMNIST from the different methods. LinearFM draws samples of visually lower quality than the rest of the methods.}
    \label{fig:bmnist_samples}
\end{figure}

\begin{table}
    \centering
    \caption{Run time measurements with varying data dimensionality $D$ and fixed batch size $B=512$. All results are reported in milliseconds as a mean with two standard deviations across 1000 iterations.}
    \label{tab:runtimeD}
    \begin{tabular}{c|c|c|c|c|c|c}
    \toprule
        Method & $D=2^2$ (ms) & $D=2^3$ (ms) & $D=2^4$ (ms) & $D=2^5$ (ms) & $D=2^6$ (ms) & $D=2^7$ (ms) \\ \midrule
        SFM & $45 \pm 3$ & $45 \pm 4$ & $48 \pm 3$ & $48 \pm 3$ & $50 \pm 3$ & $54 \pm 4$ \\ 
        \ourmethod (ALR) & $76 \pm 12$ & $74 \pm 10$ & $76 \pm 12$ & $82 \pm 12$ & $88 \pm 5$ & $85 \pm 10$ \\
        \ourmethod (ILR) & $81 \pm 12$ & $84 \pm 12$ & $83 \pm 12$ & $85 \pm 12$ & $88 \pm 6$ & $92 \pm 8$ \\ \bottomrule
    \end{tabular}
\end{table}

\begin{table}
    \centering
    \caption{Run time measurements with varying batch size $B$ and fixed data dimensionality $D=64$. All results are reported in milliseconds as a mean with two standard deviations across 1000 iterations.}
    \label{tab:runtimeB}
    \begin{tabular}{c|c|c|c|c|c|c}
    \toprule
        Method & $B=2^6$ (ms) & $B=2^7$ (ms) & $B=2^8$ (ms) & $B=2^9$ (ms) & $B=2^{10}$ (ms) & $B=2^{11}$ (ms) \\ \midrule
        SFM & $15 \pm 2$ & $18 \pm 1$ & $29 \pm 1$ & $51 \pm 3$ & $93 \pm 10$ & $183 \pm 36$ \\ 
        \ourmethod (ALR) & $13 \pm 4$ & $18 \pm 6$ & $42 \pm 6$ & $89 \pm 6$ & $171 \pm 18$ & $346 \pm 36$ \\
        \ourmethod (ILR) & $13 \pm 4$ & $18 \pm 6$ & $43 \pm 9$ & $83 \pm 12$ & $166 \pm 12$ & $339 \pm 32$ \\\bottomrule
    \end{tabular}
\end{table}

\section{ADDITIONAL EXPERIMENTS} \label{app:additional_exper}

\subsection{Estimation of the categorical probabilities}  \label{app:cat_probs}
We evaluate the accuracy of our estimator for categorical log-probabilities under the same setup as the \textit{Scalability} experiment.
We consider dimensions $D = 2, 2^2, \ldots, 2^8$ and compare the Stick-Breaking (SB) transform against SFM.
The true probabilities $\boldsymbol{p} \in \Delta^D$ and the velocity field is explained in the  Section~\ref{sec:scalability_supp}.

Recall the estimator for the discrete probabilities is given by
$\widehat{\Pr}(C=k) = \frac{q_\theta(\boldsymbol{\mu}^{(k)})}{q_\lambda(\boldsymbol{\mu}^{(k)}|\boldsymbol{e}_k)}$.
The numerator $q_\theta(\cdot)$ is computed using the Euler integrator with $200$ steps, and its divergence term is approximated via the Hutchinson trace estimator.

We compare our estimator $\widehat{\Pr}(C=k)$ to the lower bound proposed by \citet{Cheng2024}.
Figure~\ref{fig:corners2} shows that our estimator closely matches the true probabilities up to $2^5$ categories in terms of KL divergence, whereas the lower bound used by SFM deviates significantly for all $D$.
Similarly, our estimator of $p_1=0.5$ remains accurate up to $2^5$ categories, while the lower bound remains loose across all dimensions and results in numerical errors for high number of categories.
\begin{figure}
    \centering    
    \includegraphics[width=0.4\linewidth]{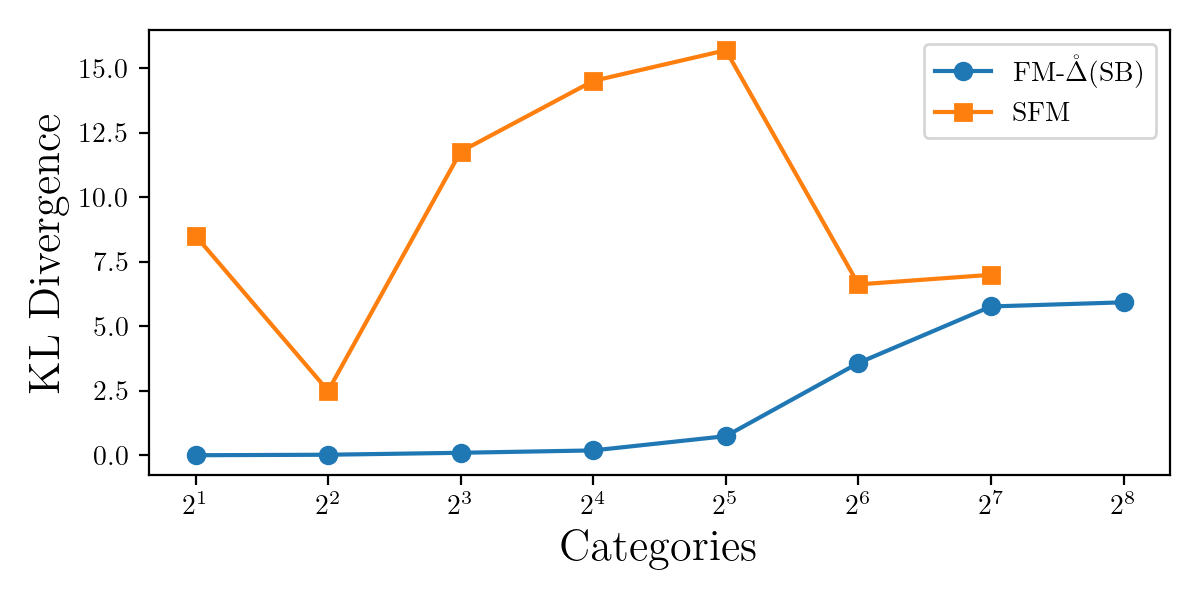}
    \includegraphics[width=0.4\linewidth]{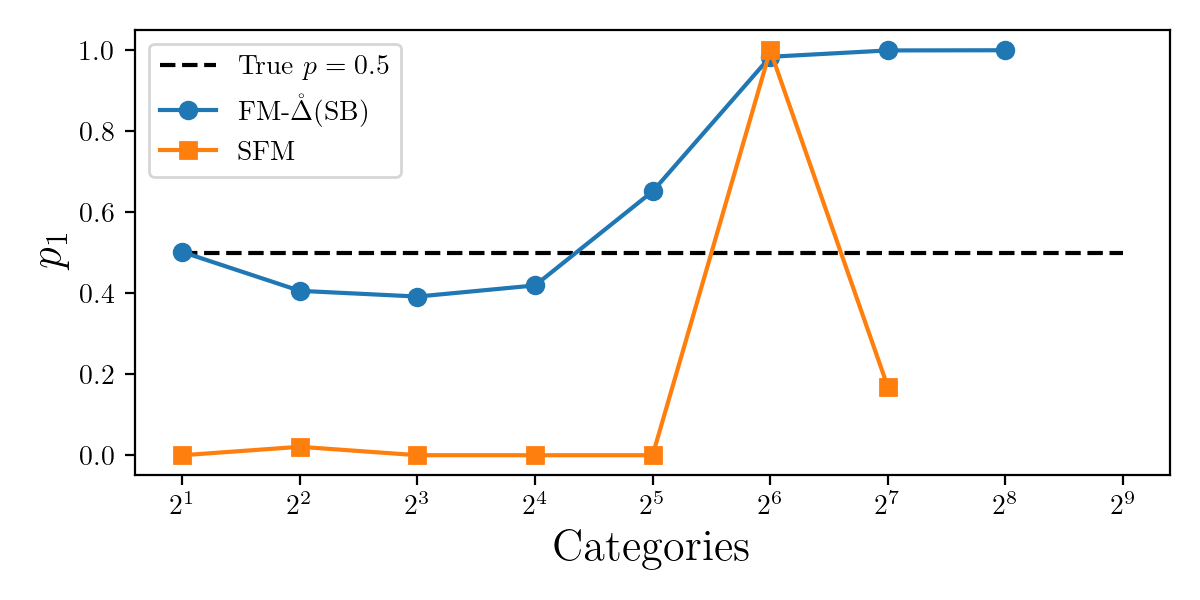}    
    \caption{    
    {Left:} KL divergence between the true categorical probabilities and either the estimation $\widehat{\Pr}(C=k)$ (blue) or lower-bound approximations given by SFM (orange).
{Right:} Estimated value and lower bound of ${p}_1$ compared to the true value $p_1=0.5$.
Missing values correspond to numerical overflows during the computation.}
    \label{fig:corners2}
\end{figure}

\subsection{Effect of the parameters} \label{app:effect_params}

Using the same setup as in Scalability, we train a velocity network across the values $\alpha\in \{1, 10, 100, \infty\}$, $\lambda\in \{0.5, 0.75, 0.99\}$, with and without scaling (see Eq.~\ref{eq:norm_mean}). We only consider the ILR bijection. The case $\alpha=1$ corresponds to the uniform distribution over each  $\arg\max$  region ($x_i>x_j$ for all $i\neq j$), and $\alpha=\infty$ to a deterministic interpolation, namely $ \lambda \boldsymbol{c} + (1-\lambda) \tfrac{1}{K}$. We generate $100,000$ samples $5$ times and plot the mean values with 2 standard deviation bands. 

Figure~\ref{fig:corners3} shows that scaling improves the accuracy when there are only 2 categories, but scaling degrades the performance once the number of categories exceeds $2^5$.
The uniform distribution ($\alpha=1$) is slightly more accurate for  $\lambda \in \{0.5, 0.75\}$ than $\lambda=0.99$. 
The performance of the method is more sensitive to the the values of $\alpha$ and $\lambda$ as $K$ increases, but we intentionally did not attempt fine-tuning $\alpha$ and $\lambda$ to maximize the empirical performance, since we want an easy-to-use pipeline.

\begin{figure}
    \centering    
    \includegraphics[width=0.49\linewidth]{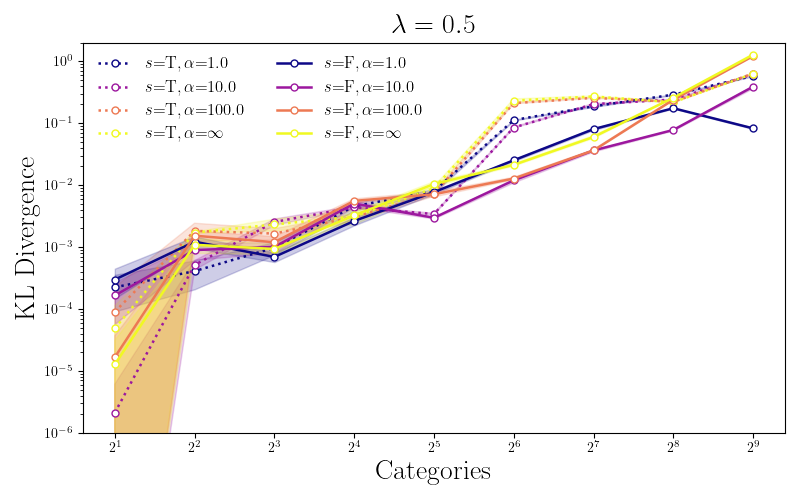}
    \includegraphics[width=0.49\linewidth]{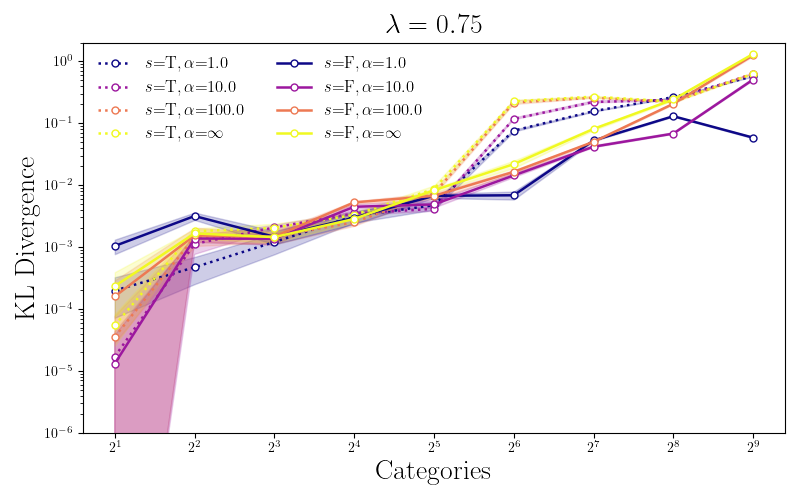}
    \includegraphics[width=0.49\linewidth]{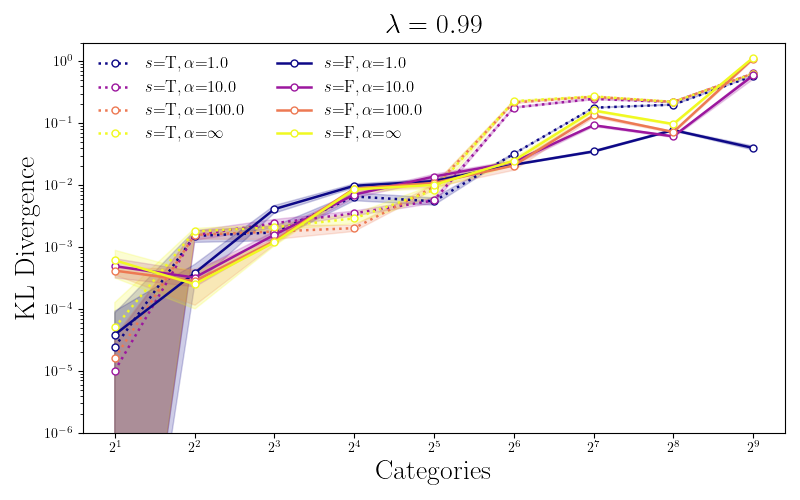}    
    \caption{The KL divergence between the true and estimated categorical probabilities for $\alpha\in \{1, 10, 100, \infty\}$, $\lambda\in \{0.5, 0.75, 0.99\}$, with and without scaling. The bands are two standard deviations away from the mean.     
    }
    \label{fig:corners3}
\end{figure}

\subsection{Run Times}

We compare the run times between SFM and \ourmethod (ALR) with the same setup as the \emph{Scalability} experiment. To measure the effect of the number of categories, we consider dimensions  $D = 2^2, \ldots, 2^7$ with a fixed batch size of $B=512$, presented in Table~\ref{tab:runtimeD}. Then, to measure the effect of batch size, we consider batch sizes $B=2^6, 2^7, \ldots, 2^11$ with a fixed $D=64$, presented in Table~\ref{tab:runtimeB}. All the run times are reported with mean and two standard deviations over 1000 training iterations in machines with Xeon E5 2680 v3 2.50GHz CPUs, we do not use GPUs for these measurements.

\subsection{Numerical results for Additive Logratio and Multiplicative Logratio bijections}  \label{app:alr_mlr}

In Section~\ref{sec:bij} we defined ALR and MLR as possible bijections, before introducing our proposed bijections SB and ILR. It would still be useful to clarify how well ALR and MLR would work.

We now ran the Checkerboard, BMNIST and DNA experiments with the ALR and MLR mappings, reporting the results in Table \ref{tab:alr_mlr}. SB remains the best on DNA and Checkerboard and ILR on binarized MNIST, but both ALR and MLR also work relatively well. SB and MLR are highly similar (for BMNIST with 2 classes they are exactly the same map), as they should because the former mostly just stabilizes the computation by centering.

\begin{table} 
\centering
\caption{Numerical results for ALR and MLR bijections.}
\label{tab:alr_mlr}
\begin{tabular}{lccc} 
\hline
Method & Checkerboard & BMNIST & DNA \\
\hline
FM-$\mathring{\Delta}$(ALR)     & 6.4\%      & 4.98 & 0.029 \\
FM-$\mathring{\Delta}$(ALR)w/OT & --         & 5.22 & 0.037 \\
FM-$\mathring{\Delta}$(MLR)     & 5.7\%      & 4.93 & 0.027 \\
FM-$\mathring{\Delta}$(MLR)w/OT & --         & 4.51 & 0.024 \\
FM-$\mathring{\Delta}$(SB)      & \textbf{5.4\%} & 4.93 & 0.028 \\
FM-$\mathring{\Delta}$(SB)w/OT  & --         & 4.51 & \textbf{0.021} \\
FM-$\mathring{\Delta}$(ILR)     & 6.8\%      & \textbf{4.36} & 0.026 \\
FM-$\mathring{\Delta}$(ILR)w/OT & --         & 4.57 & 0.022 \\
\hline
\end{tabular}
\end{table}

\end{document}